\documentclass[conference]{IEEEtran}
\usepackage{times}

\usepackage{multicol}
\usepackage[bookmarks=true,colorlinks=true]{hyperref}

\usepackage{comment}
\usepackage{siunitx}
\usepackage{relsize}
\usepackage{ifthen}
\usepackage[colorinlistoftodos]{todonotes}

\usepackage[caption=false]{subfig}

\usepackage[vlined,ruled,linesnumbered]{algorithm2e}
\usepackage{graphics} % for pdf, bitmapped graphics files
\usepackage{rotating}
\usepackage{color}
\usepackage{enumerate}
\usepackage[T1]{fontenc}
\usepackage{psfrag}
\usepackage{epsfig} % for postscript graphics files
\usepackage{booktabs}
\usepackage{graphicx,url}
\usepackage{multirow}
\usepackage{array}
\usepackage{latexsym}
\usepackage{amsfonts}
\usepackage{amsmath}
\usepackage{amssymb}
\usepackage{mathtools}
\usepackage{xstring}
\usepackage[noend]{algorithmic}
\usepackage{multirow}
\usepackage{xcolor}
\usepackage{prettyref}
\usepackage{flexisym}
\usepackage{bigdelim}
\usepackage{breqn} % load this last
\usepackage{listings}
\let\labelindent\relax
\usepackage{enumitem}
\usepackage{xspace}
\usepackage{bm}
\graphicspath{{./figures/}}
\usepackage{tikz}
\usetikzlibrary{matrix,calc}

\usepackage{mdwlist}

\let\itemize\undefined
\makecompactlist{itemize}{stditemize}

\usepackage{amsthm}
\newrefformat{prob}{Problem\,\ref{#1}}
\newrefformat{def}{Definition\,\ref{#1}}
\newrefformat{sec}{Section\,\ref{#1}}
\newrefformat{sub}{Section\,\ref{#1}}
\newrefformat{prop}{Proposition\,\ref{#1}}
\newrefformat{app}{Appendix\,\ref{#1}}
\newrefformat{alg}{Algorithm\,\ref{#1}}
\newrefformat{cor}{Corollary\,\ref{#1}}
\newrefformat{thm}{Theorem\,\ref{#1}}
\newrefformat{lem}{Lemma\,\ref{#1}}
\newrefformat{fig}{Fig.\,\ref{#1}}
\newrefformat{tab}{Table\,\ref{#1}}

\newtheorem{theorem}{Theorem}

\newtheorem{remark}[theorem]{Remark}
\newtheorem{example}{Example}

\newcommand{\cf}{\emph{cf.}\xspace}

\newcommand{\bdmath}{\begin{dmath}}
\newcommand{\edmath}{\end{dmath}}
\newcommand{\beq}{\begin{equation}}
\newcommand{\eeq}{\end{equation}}
\newcommand{\bdm}{\begin{displaymath}}
\newcommand{\edm}{\end{displaymath}}
\newcommand{\bea}{\begin{eqnarray}}
\newcommand{\eea}{\end{eqnarray}}
\newcommand{\beal}{\beq \begin{array}{ll}}
\newcommand{\eeal}{\end{array} \eeq}
\newcommand{\beas}{\begin{eqnarray*}}
\newcommand{\eeas}{\end{eqnarray*}}
\newcommand{\ba}{\begin{array}}
\newcommand{\ea}{\end{array}}
\newcommand{\bit}{\begin{itemize}}
\newcommand{\eit}{\end{itemize}}
\newcommand{\ben}{\begin{enumerate}}
\newcommand{\een}{\end{enumerate}}

\newcommand{\calT}{{\cal T}}

\newcommand{\eg}{\emph{e.g.,}\xspace}
\newcommand{\ie}{\emph{i.e.,}\xspace}

\newcommand{\hide}[1]{}
\newcommand{\wrt}{w.r.t.\xspace}

\newcommand{\hiddenText}{{\color{gray} hidden text.}}
\newcommand{\hideWithText}[1]{\hiddenText}

\newcommand{\Real}[1]{ { {\mathbb R}^{#1} } }

\newcommand{\scenario}[1]{{\smaller \sf#1}\xspace}

\newcommand{\blue}[1]{{\color{blue}#1}}

\newcommand{\linkToPdf}[1]{\href{#1}{\blue{(pdf)}}}
\newcommand{\linkToPpt}[1]{\href{#1}{\blue{(ppt)}}}
\newcommand{\linkToCode}[1]{\href{#1}{\blue{(code)}}}
\newcommand{\linkToWeb}[1]{\href{#1}{\blue{(web)}}}
\newcommand{\linkToVideo}[1]{\href{#1}{\blue{(video)}}}
\newcommand{\linkToMedia}[1]{\href{#1}{\blue{(media)}}}
\newcommand{\award}[1]{\xspace} % {{\red{#1}}} % omit awards

\usepackage{adjustbox}

\usepackage[backend=biber,
            style=numeric,
            maxnames=99,   % bibliography
            maxcitenames=3 % citations
           ]{biblatex}
\AtEveryBibitem{\clearfield{note}}
\AtEveryBibitem{\clearfield{doi}}
\AtEveryBibitem{\clearfield{url}}
\newcommand{\KL}{D_\mathrm{KL}}
\newcommand{\supp}{{Supplementary Material}}
\newcommand{\jax}{\textsc{JAX}}
\newcommand{\blackjax}{\textsc{BlackJax}}
\newcommand{\tsmc}{\scenario{TSMC}}
\newcommand{\pnuts}{\scenario{Parallel-NUTS}}
\newcommand{\pmala}{\scenario{Parallel-MALA}}
\newcommand{\pmppi}{\scenario{Parallel-MPPI}}
\newcommand{\pipopt}{\scenario{Parallel-IPOPT}}
\newcommand{\casadi}{\textsc{CasADi}}
\newcommand{\ppo}{\scenario{PPO}}
\newcommand{\sac}{\scenario{SAC}}
\newcommand{\position}{\text{pos}}
\newcommand{\velocity}{\text{vel}}
\newcommand{\distance}{\text{dist}}
\newcommand{\tsmcd}{\scenario{TSMC:D}}
\newcommand{\tsmce}{\scenario{TSMC:E}}
\newcommand{\warn}[1]{{#1}}

\begin{document}

% paper title
% IEEEtran uses a large default title font; override it.
% 19pt is between \LARGE and \Huge (10pt base); 23pt is the baselineskip.
% \title{{\fontsize{23}{23}\selectfont Tempered Sequential Monte Carlo for Trajectory and Policy Optimization with Differentiable Dynamics}\vspace{-2mm}}

\title{{\fontsize{23}{23}\selectfont Tempered Sequential Monte Carlo for Trajectory and Policy Optimization with Differentiable Dynamics\vspace{-2mm}}}

% You will get a Paper-ID when submitting a pdf file to the conference system
% \author{Author Names Omitted for Anonymous Review. Paper-ID 405}
\author{Heng Yang \\ Harvard University \\ \texttt{https://github.com/ComputationalRobotics/TSMC} \vspace{-4mm}}

% \author{\authorblockN{Heng Yang}
% \authorblockA{Harvard University \vspace{-4mm}}}

%\and
%\authorblockN{Homer Simpson}
%\authorblockA{Twentieth Century Fox\\
%Springfield, USA\\
%Email: homer@thesimpsons.com}
%\and
%\authorblockN{James Kirk\\ and Montgomery Scott}
%\authorblockA{Starfleet Academy\\
%San Francisco, California 96678-2391\\
%Telephone: (800) 555--1212\\
%Fax: (888) 555--1212}}

% avoiding spaces at the end of the author lines is not a problem with
% conference papers because we don't use \thanks or \IEEEmembership

% for over three affiliations, or if they all won't fit within the width
% of the page, use this alternative format:
% 
%\author{\authorblockN{Michael Shell\authorrefmark{1},
%Homer Simpson\authorrefmark{2},
%James Kirk\authorrefmark{3}, 
%Montgomery Scott\authorrefmark{3} and
%Eldon Tyrell\authorrefmark{4}}
%\authorblockA{\authorrefmark{1}School of Electrical and Computer Engineering\\
%Georgia Institute of Technology,
%Atlanta, Georgia 30332--0250\\ Email: mshell@ece.gatech.edu}
%\authorblockA{\authorrefmark{2}Twentieth Century Fox, Springfield, USA\\
%Email: homer@thesimpsons.com}
%\authorblockA{\authorrefmark{3}Starfleet Academy, San Francisco, California 96678-2391\\
%Telephone: (800) 555--1212, Fax: (888) 555--1212}
%\authorblockA{\authorrefmark{4}Tyrell Inc., 123 Replicant Street, Los Angeles, California 90210--4321}}

\maketitle

%!TEX root = ../main.tex

\begin{abstract}
    We propose a sampling-based framework for finite-horizon trajectory and policy optimization under differentiable dynamics by casting controller design as inference. Specifically, we minimize a KL-regularized expected trajectory cost, which yields an optimal ``Boltzmann-tilted'' distribution over controller parameters that concentrates on low-cost solutions as temperature decreases. To sample efficiently from this sharp, potentially multimodal target, we introduce \emph{tempered sequential Monte Carlo} (TSMC): an annealing scheme that adaptively reweights and resamples particles along a tempering path from a prior to the target distribution, while using Hamiltonian Monte Carlo rejuvenation to maintain diversity and exploit \emph{exact gradients} obtained by differentiating through trajectory rollouts. For policy optimization, we extend TSMC via (i) a deterministic empirical approximation of the initial-state distribution and (ii) an extended-space construction that treats rollout randomness as auxiliary variables. Experiments across trajectory- and policy-optimization benchmarks show that TSMC is broadly applicable and compares favorably to state-of-the-art baselines.
\end{abstract}

\IEEEpeerreviewmaketitle

%!TEX root = ../main.tex

\section{Introduction}
\label{sec:introduction}

Consider a discrete-time dynamical system
\begin{equation}
    x_{t+1} = f(x_t, u_t), \quad t \in \mathbb{N}
\end{equation}
where \(x_t \in \Real{n}\) is the state and \(u_t \in \Real{m}\) is the control input. We focus on deterministic dynamics for clarity of exposition and discuss extensions to stochastic dynamics in Remark~\ref{remark:stochastic-dynamics}. For a finite horizon \(T\), we denote a state-control trajectory by $\tau = (x_0, u_0, x_1, u_1, \ldots, x_{T-1}, u_{T-1}, x_T)$, with total cost
\begin{equation}
    J(\tau) = \sum_{t=0}^{T-1} \ell_t(x_t, u_t) + \ell_T(x_T).
\end{equation}
Here \(\ell_t\) for \(t=0,\dots,T-1\) are stage costs and \(\ell_T\) is a terminal cost. Let \(\mu\) denote a distribution over initial states \(x_0\), and let \(\theta \in \Real{d}\) parameterize a (state-feedback) controller \(\pi_\theta:\Real{n}\to\Real{m}\). Under \(\pi_\theta\), the control input is \(u_t = \pi_\theta(x_t)\), which together with the dynamics induces a trajectory rollout \(\tau(x_0,\theta)\). For simplicity, we focus on deterministic controllers. We then consider the unified objective
\begin{equation}\label{eq:trajectory-policy-optimization}
\min_{\theta \in \Real{d}} \mathbb{E}_{x_0 \sim \mu} \left[ J(\tau(x_0, \theta)) \right].
\end{equation}
Problem \eqref{eq:trajectory-policy-optimization}, henceforth referred to as \emph{trajectory and policy optimization} (TPO), subsumes two fundamental settings in optimal control and reinforcement learning:
\begin{itemize}
    \item \textbf{Trajectory Optimization (TO)} \cite{rawlings20book-mpc,bryson69book-optimalcontrol,conway10book-spacecraft,kang24wafr-strom,posa14ijrr-direct}: fix the initial state by setting \(\mu=\delta_{x_0}\) (a Dirac measure), and choose an open-loop parameterization \(\theta \!=\! (u_0,\dots,u_{T-1})\) (equivalently, \(\pi_\theta(x_t)\equiv u_t\)). Then \eqref{eq:trajectory-policy-optimization} reduces to optimizing a single control sequence for a given \(x_0\). 
    
    % In practice, TO is often solved repeatedly in a receding-horizon fashion to yield an implicit state-feedback policy, \ie model predictive control \cite{rawlings20book-mpc,wensing23tro-optimization}.
    \item \textbf{Policy Optimization (PO)} \cite{sutton98book-rlbook,schulman15icml-trpo,schulman17arxiv-ppo,kakade01neurips-npg,haarnoja18icml-sac}: let \(\mu\) represent a distribution over initial conditions (\eg uniform over a region in the state space), and parameterize \(\pi_\theta\) as a state-feedback policy (\eg with a neural network). Then \eqref{eq:trajectory-policy-optimization} becomes the problem of minimizing expected trajectory cost over the initial-state distribution.
\end{itemize}

Motivated by the recent efforts in \emph{differentiable} physics simulation---particularly contact solvers \cite{le25arxiv-highly,hu20iclr-difftaichi,schwarke25corl-learning,freeman21arxiv-brax,paulus25arxiv-hard,newbury24ieee-review,howell22arxiv-dojo}---and in learning \emph{differentiable} world models \cite{qi25arxiv-gpc,li25arxiv-rwm,zhang25rss-particle,ai25sr-review,amigo25corl-forl,agarwal25arxiv-cosmos,wu23corl-daydreamer}, we assume the dynamics \(f(x_t,u_t)\) is differentiable in both \(x_t\) and \(u_t\). This enables gradient-based approaches that ``exploit'' differentiability: in TO, nonlinear programming techniques (single or multiple shooting) optimize an open-loop control sequence \cite{rawlings20book-mpc,andersson19mpc-casadi,verschueren22mpc-acados}; in PO, recent ``first-order'' methods \cite{clavera20iclr-model,xu22iclr-accelerated,amigo25corl-forl,amos21l4dc-model,wiedemann23icra-apg} leverage differentiability to compute policy gradients more efficiently and improve sample efficiency. A drawback, however, is that gradient-based optimization can get trapped in undesirable local minima, and gradients from differentiable simulators can be uninformative in the presence of nonsmooth dynamics \cite{suh22icml-differentiable}. In contrast, sampling-based methods such as \emph{model predictive path integral control} (MPPI) \cite{williams15arxiv-mppi,williams17jgcd-mppi} ``explore'' the control space and use the dynamics model only to evaluate trajectory costs; they can be less sensitive to poor local minima but often require many samples. This motivates our central question: \emph{can we design a unified and principled algorithm for TPO that combines the benefits of sampling and gradient-based optimization under differentiable dynamics?}

\textbf{From Optimization to Inference.} Towards this, instead of optimizing a single parameter \(\theta\) in \eqref{eq:trajectory-policy-optimization}, we consider optimizing a distribution \(p\) over \( \theta \in \Real{d}\). Let \(p_0\) be a prior distribution with \emph{full support} on \(\Real{d}\) (\ie with density \(p_0(\theta)>0\) for all \(\theta\in\Real{d}\); \eg a Gaussian). We study the KL-regularized problem
\begin{equation}\label{eq:regularized-trajectory-policy-optimization}
\min_{p} \mathbb{E}_{x_0 \sim \mu, \theta \sim p} \left[ J(\tau(x_0, \theta)) \right] + \lambda \KL(p \| p_0),
\end{equation}
where \(\KL(p \| p_0) := \int_{\Real{d}} p(\theta) \log \frac{p(\theta)}{p_0(\theta)} d\theta\) is the Kullback--Leibler (KL) divergence and \(\lambda>0\) controls the strength of regularization (equivalently, a ``temperature''). At $\lambda=0$, problem \eqref{eq:regularized-trajectory-policy-optimization} is equivalent to the TPO problem \eqref{eq:trajectory-policy-optimization}: they have the same optimal value and any distribution supported on the optimal solution set of \eqref{eq:trajectory-policy-optimization} is optimal for \eqref{eq:regularized-trajectory-policy-optimization}. As \(\lambda \to 0^+\), the optimal \(p^\star\) concentrates its mass on minimizers of \eqref{eq:trajectory-policy-optimization}. A key advantage of \eqref{eq:regularized-trajectory-policy-optimization} is that it admits a closed-form solution. 

\begin{theorem}[Optimal Distribution]\label{thm:optimal-distribution}
    Assume \(\ell_t(\cdot,\cdot)\ge 0\) for \(t=0,\dots,T-1\) and \(\ell_T(\cdot)\ge 0\). Define the energy function
    \begin{equation}\label{eq:energy-function}
        E(\theta) = \mathbb{E}_{x_0 \sim \mu} \left[ J(\tau(x_0, \theta)) \right],
    \end{equation}
    then the optimal distribution \(p^\star\) of \eqref{eq:regularized-trajectory-policy-optimization} is given by
    \begin{equation}\label{eq:optimal-distribution}
        p^\star(\theta) = \frac{1}{Z}\, p_0(\theta) \exp\left( - \frac{1}{\lambda} E(\theta) \right),
    \end{equation}
    where 
    \begin{equation}\label{eq:partition-function}
    Z = \int_{\Real{d}} p_0(\theta) \exp\left( - \frac{1}{\lambda} E(\theta) \right) d\theta
    \end{equation}
    is the normalizing constant (partition function).
\end{theorem}

\warn{A proof is provided in the \supp.}
Theorem~\ref{thm:optimal-distribution} reframes optimization as inference: the optimal solution of \eqref{eq:regularized-trajectory-policy-optimization} is obtained by \emph{Boltzmann-tilting} the prior \(p_0\) with the energy \(E(\theta)\), namely \(p^\star(\theta)\propto p_0(\theta)\exp(-E(\theta)/\lambda)\) (a.k.a. the Gibbs posterior \cite{mackay03book-information,jaynes57physrev-information}). Intuitively, this tilt exponentially \emph{upweights} controller parameters \(\theta\) that incur lower energy, \ie lower expected trajectory cost. Thus, if we can efficiently sample from the corresponding (unnormalized) density---especially at low temperature \(\lambda\)---then samples concentrate near minimizers of the original TPO objective \eqref{eq:trajectory-policy-optimization}.
This ``control-as-inference'' characterization is not new: closely related forms appear in information-theoretic optimal control \cite{todorov06neurips-linearly,kappen05jsm-path,williams17jgcd-mppi,williams15arxiv-mppi} and reinforcement learning \cite{schulman15icml-trpo,haarnoja18icml-sac,levine18arxiv-controlasinference,ma25icml-efficient}. 

While Theorem~\ref{thm:optimal-distribution} gives a closed-form solution to the KL-regularized TPO problem, sampling from the Boltzmann-tilted distribution \eqref{eq:optimal-distribution} remains challenging. At low temperature, the target distribution can be sharply concentrated and highly multimodal, reflecting the nonconvexity of the underlying TPO objective \eqref{eq:trajectory-policy-optimization}. To the best of our knowledge, sampling from such hard targets has received limited attention in the optimal control and reinforcement learning literature. Existing approaches are typically based on importance sampling (\eg MPPI and the cross-entropy method \cite{de05aor-cem}) or on Markov chain Monte Carlo (MCMC) (\eg Bayesian motion planning \cite{zhi23icraw-bayesian,kondic24smthesis-monte} and energy-based RL policies \cite{haarnoja17icml-energy-based,sallans20neurips-energy}). For sharp, multimodal targets, the former often requires many samples, while the latter can suffer from long mixing times.

\textbf{Contribution.} We develop an algorithm for sampling from the Boltzmann-tilted distribution \eqref{eq:optimal-distribution} that combines sampling-based exploration with gradient-based optimization, and applies to both TO and PO. Our approach builds on \emph{tempered sequential Monte Carlo} (TSMC), a state-of-the-art tool for multimodal, low-temperature targets in Bayesian inference \cite{neal01sc-ais,doucet01book-smc,delmoral06jrsc-smc,dai22jasa-smc}. TSMC introduces a \emph{tempering path} that gradually deforms an easy-to-sample prior into the desired Boltzmann tilt (\cf Fig.~\ref{fig:tsmc-overview}); at each tempering level, it reweights and resamples particles to focus computation on low-energy regions, then applies Hamiltonian Monte Carlo (HMC) rejuvenation to restore diversity and enable local exploration while preserving the tempered target. We show that TSMC can be applied directly to TO, where \(\nabla_\theta E(\theta)\) is available via automatic differentiation through rollouts. For PO, we introduce practical variants based on (i) a deterministic approximation of \(E(\theta)\) and (ii) an extended-space construction that augments each particle with a batch of initial conditions. Experiments across TO and PO benchmarks demonstrate robust performance and improved results relative to existing methods.

\textbf{Outline.} Section~\ref{sec:tsmc} provides a tutorial-style introduction to TSMC. Sections~\ref{sec:tsmc-to} and \ref{sec:tsmc-po} apply TSMC to trajectory optimization and policy optimization, respectively, with experimental results in Sections~\ref{sec:experiments-to} and \ref{sec:experiments-po}. Section~\ref{sec:conclusion} concludes with limitations and future directions. \supp{} contains proofs, additional results, and related work.

\begin{figure*}[t]
    \centering
    \includegraphics[width=\textwidth]{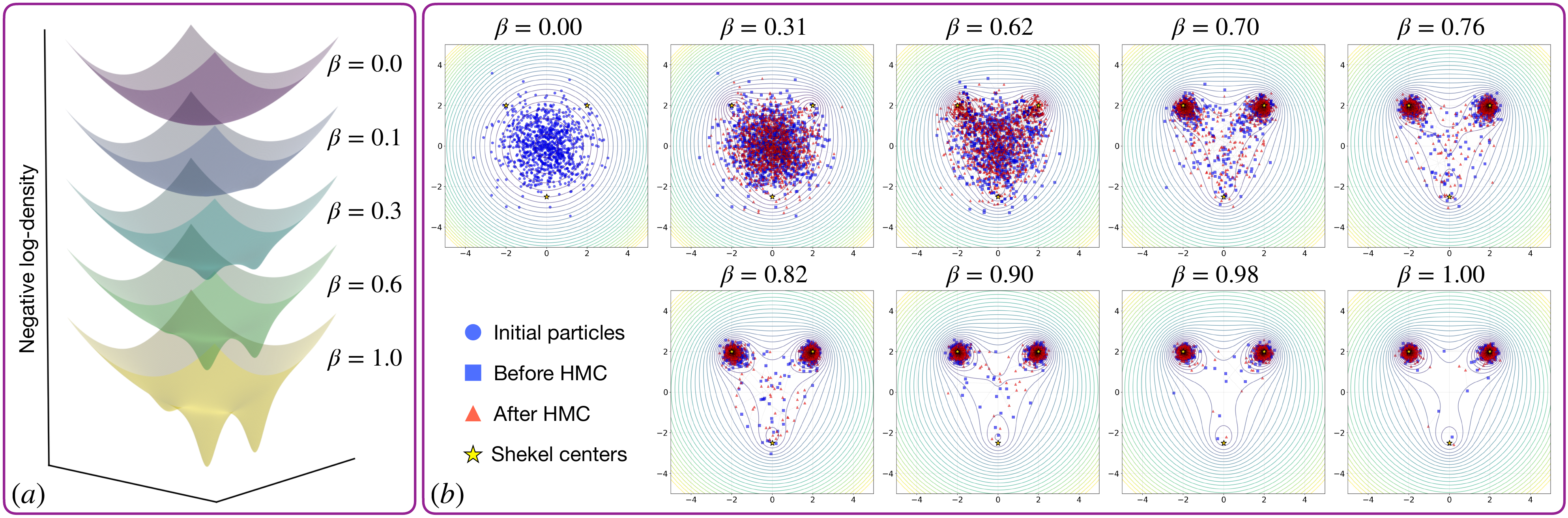}
    \vspace{-7mm}
    \caption{Tempered Sequential Monte Carlo samples from the Shekel-tilted distribution in Example~\ref{ex:shekel-tilted-distribution}. (a) Negative log-density of the Shekel-tilted distribution at different tempering values. (b) Particle locations before and after HMC rejuvenation on contours of the Shekel-tilted negative log-density.}
    \label{fig:tsmc-overview}
    \vspace{-6mm}
\end{figure*}
%!TEX root = ../main.tex

\section{Tempered Sequential Monte Carlo}
\label{sec:tsmc}

Recall that our goal is to sample from the Boltzmann-tilted distribution \eqref{eq:optimal-distribution}, where the energy function \(E(\theta)\) is differentiable in \(\theta\). This section introduces the key ingredients of tempered sequential Monte Carlo (TSMC). For concreteness, we use the following running example. 
% which is both multimodal and sharply concentrated.

\begin{example}[Shekel-tilted Distribution]\label{ex:shekel-tilted-distribution}
Let $\theta \in \Real{2}$ and consider the energy function defined by the negative Shekel function \cite{shekel71-shekel} with three centers:
    \begin{equation}
        E(\theta) = - \sum_{i=1}^3 \frac{1}{\Vert \theta - c_i \Vert^2 + s_i},
    \end{equation}
    where $c_1 = (2,2)$, $c_2 = (-2,2)$, and $c_3 = (0,-2.5)$ are the centers, and $s_1 = s_2 = 0.5$, $s_3 = 1.2$ are the widths. This nonconvex landscape has three local minima at the three centers, with $c_1$ and $c_2$ being global minima.
    Let $p_0$ be the standard Gaussian, and define the Shekel-tilted distribution by
    \begin{equation}
        p(\theta) \propto \exp\left(-\left(\frac{1}{2}\Vert \theta \Vert^2 + \frac{1}{\lambda}E(\theta) \right)\right).
    \end{equation}
    The lowest surface in Fig.~\ref{fig:tsmc-overview}(a) plots the negative log-density $\frac{1}{2}\Vert \theta \Vert^2 + \frac{1}{\lambda}E(\theta)$ with $\lambda=0.1$, revealing multimodality of $p(\theta)$.
\end{example}

\subsection{Tempered Importance Sampling}
\label{subsec:tsmc-tempered-importance-sampling}

To address multimodality of the Boltzmann-tilted distribution \eqref{eq:optimal-distribution}, TSMC introduces a \emph{tempering path} that gradually deforms the prior \(p_0\) into the target \(p^\star\). Concretely, choose an increasing sequence of scalars $0 =\beta_0 < \beta_1 < \cdots < \beta_K = 1$ and define intermediate distributions
\begin{equation}\label{eq:tempered-path}
p_k(\theta)\ \propto  p_0(\theta)\exp\!\left(-\frac{\beta_k}{\lambda} E(\theta) \right), \quad k=1,\dots,K.
\end{equation}
Clearly, $p_0$ coincides with the prior, and $p_K$ coincides with the target $p^\star$. 
Fig.~\ref{fig:tsmc-overview}(a) plots the negative log-density of $p_k(\theta)$ for the Shekel-tilted distribution at different values of $\beta$. As $\beta$ increases, the nonconvex Shekel term is weighted more heavily relative to the convex quadratic, smoothly deforming the landscape from convex to strongly nonconvex. 
To tractably sample from $p_k$, we maintain a weighted particle approximation \(\{(\theta_k(i), w_k(i))\}_{i=1}^N\) with normalized weights $\sum_{i=1}^N w_k(i)=1$. Here \(k\) indexes the tempering step and \(i\) indexes the \(i\)-th particle (and its weight).

\textbf{Importance Sampling.} Suppose we have a weighted particle approximation \(\{(\theta_{k-1}(i), w_{k-1}(i))\}_{i=1}^N\) for \(p_{k-1}\). To update it to approximate \(p_k\), we reweight each particle by the \emph{incremental likelihood ratio}
\begin{equation}\label{eq:tsmc-incremental-weight}
\Delta w_k(i)  = \exp\!\Big(-\frac{\beta_k-\beta_{k-1}}{\lambda}E(\theta_{k-1}(i))\Big)
\end{equation}
and then renormalize,
\begin{equation}\label{eq:tsmc-weight-update}
w_k(i) = \frac{w_{k-1}(i) \cdot \Delta w_k(i)}{\sum_{j=1}^N w_{k-1}(j) \cdot \Delta w_k(j)}, \qquad i=1,\dots,N.
\end{equation}
The factor \(\Delta w_k(i)\) is (up to a constant) the ratio \(p_k(\theta(i))/p_{k-1}(\theta(i))\), so particles with lower energy \(E(\theta)\) are exponentially favored (upweighted) as \(\beta\) increases.

\textbf{Adaptive Tempering.}
As \(\beta_k\) increases, weight disparity grows and the particle approximation may degenerate, commonly tracked by the effective sample size (ESS),
\begin{equation}\label{eq:ess}
\mathrm{ESS}_k \triangleq \frac{1}{\sum_{i=1}^N (w_k(i))^2}\ \in\ [1,N],
\end{equation}
where the bounds correspond to complete degeneracy (one particle has all weight) and uniform weights ($w_k(i)=1/N, i=1,\dots,N$), respectively. In practice, we use ESS to choose the next temperature \(\beta_k\) adaptively: given particles from \(p_{k-1}\), we pick \(\beta_k\in(\beta_{k-1},1]\) (\eg using bisection) so that the ESS \emph{after} reweighting by \eqref{eq:tsmc-weight-update} hits a target level, such as \(\mathrm{ESS}_k \approx \rho N\) for a fixed \(\rho\in(0,1)\). This yields an adaptive tempering schedule that prevents abrupt weight collapse. For the Shekel-tilted distribution, Fig.~\ref{fig:tsmc-overview}(b) shows the resulting adaptive schedule \(\{\beta_k\}_{k=0}^K\) obtained by targeting \(\mathrm{ESS}_k \approx \rho N\) with \(\rho=0.9\).

\textbf{Resampling.} After selecting \(\beta_k\) and computing the updated weights via \eqref{eq:tsmc-weight-update}, the particle approximation of \(p_k\) is \emph{weighted}. Resampling converts this weighted set into an \emph{approximately unweighted} one, which (i) prevents gradual accumulation of weight disparity across steps and (ii) focuses computation on particles that have non-negligible mass under \(p_k\). Concretely, we draw ancestor indices \(a_1,\dots,a_N\) i.i.d.\ from the categorical distribution with probabilities \((w_k(1),\dots,w_k(N))\), set \(\theta_k(i) \leftarrow \theta_{k-1}(a_i)\), and reset weights to \(w_k(i)=1/N\).

Repeated reweighting and resampling can reduce particle diversity: the population may collapse onto a small number of distinct particles that carry most of the mass. Moreover, without additional moves, the particle locations remain restricted to the initial sample \(\{\theta_0(i)\}_{i=1}^N\) and cannot explore new parameter values. Next, we incorporate MCMC transitions to rejuvenate the particle set and enable local exploration while preserving the tempered target distribution.

\subsection{MCMC Rejuvenation}
\label{subsec:tsmc-mcmc-rejuvenation}

At each tempering level \(k\) and after resampling, we apply a Markov transition kernel \(\calT_k(\theta'\mid\theta)\) to each particle. The key requirement is that \(\calT_k\) leaves \(p_k\) \emph{invariant}, meaning that if \(\theta \sim p_k\) and \(\theta'\sim \calT_k(\cdot\mid\theta)\), then \(\theta' \sim p_k\). Equivalently,
\begin{equation}\label{eq:mcmc-invariance}
\int p_k(\theta)\,\calT_k(\theta'\mid\theta)\,d\theta \ =\ p_k(\theta').
\end{equation}
Thus, rejuvenation can be viewed as a ``move'' step that improves diversity without biasing the particle approximation. 
For illustration, Fig.~\ref{fig:tsmc-overview}(b) overlays particle locations before and after MCMC rejuvenation on contours of the Shekel-tilted negative log-density. The rejuvenation step proposes distant moves, enabling exploration of the parameter space.

\textbf{Hamiltonian Monte Carlo.} There are many MCMC kernels that can be used for rejuvenation. Here we focus on Hamiltonian Monte Carlo (HMC) \cite{duane87plb-hmc,neal11chapter-hmc,betancourt17arxiv-hmc}, which is a popular choice for its ability to propose distant moves (and hence improve particle diversity).
For each tempering level \(k\), define the (unnormalized) negative log-density (potential)
\begin{equation}\label{eq:tsmc-potential}
V_k(\theta)\ \triangleq\ \frac{\beta_k}{\lambda}E(\theta)\ -\ \log p_0(\theta),
\end{equation}
so that \(p_k(\theta)\propto \exp\!\big(-V_k(\theta)\big)\) (\cf the definition in \eqref{eq:tempered-path}). HMC augments \(\theta\) with an auxiliary momentum \(r\in\Real{d}\) drawn from a Gaussian \(r\sim\mathcal{N}(0,M)\), where \(M\succ 0\) is a mass matrix, and defines the Hamiltonian
$H_k(\theta,r) \triangleq V_k(\theta) + 0.5r^\top M^{-1}r$.
HMC proposes distant moves by approximately simulating Hamiltonian dynamics under \(H_k\) using a symplectic leapfrog integrator \cite{hairer06book-gni}. In particular, given step size \(\varepsilon>0\) and number of steps \(L\), a single HMC transition proceeds as follows:
\begin{itemize}
    \item \emph{Refresh momentum:} sample \(r_0 \sim \mathcal{N}(0,M)\); set \(\theta_0=\theta\).
    \item \emph{Leapfrog integration:} for \(s=0,\dots,L-1\),
    \begin{equation}\label{eq:hmc-leapfrog}
    \begin{aligned}
    r_{s+\frac{1}{2}} &= r_s - \tfrac{\varepsilon}{2}\nabla_\theta V_k(\theta_s),\\
    \theta_{s+1} &= \theta_s + \varepsilon\,M^{-1}r_{s+\frac{1}{2}},\\
    r_{s+1} &= r_{s+\frac{1}{2}} - \tfrac{\varepsilon}{2}\nabla_\theta V_k(\theta_{s+1}).
    \end{aligned}
    \end{equation}
    \item \emph{Metropolis correction:} Define $
(\theta', r') = (\theta_L, -r_L)$ and accept the proposal with probability
\begin{equation}\label{eq:hmc-accept}
\min\left\{1,\exp\!\big(-H_k(\theta',r') + H_k(\theta_0,r_0)\big)\right\}.
\end{equation}
Otherwise set $\theta'=\theta_0$. Return the new position $\theta'$.
\end{itemize}
Intuitively, invariance follows because leapfrog defines a reversible, volume-preserving proposal map for the Hamiltonian dynamics \cite{neal11chapter-hmc}, and the Metropolis--Hastings accept/reject step corrects the discretization error to enforce detailed balance with respect to the target density \(p_k\) \cite{robert99book-mc}. Detailed balance implies that \(\calT_k\) leaves \(p_k\) invariant, \ie \eqref{eq:mcmc-invariance} holds. 

\textbf{Parameter Tuning.} The HMC kernel has three key parameters: the mass matrix \(M\), the step size \(\varepsilon\), and the number of leapfrog steps \(L\). Ideally, \(M\) matches the local scaling/correlation structure of the target (so different coordinates evolve on comparable time scales), \(\varepsilon\) is as large as possible while maintaining a reasonable acceptance probability, and the trajectory length \(L\varepsilon\) is long enough to make a substantive move without wasting computation by retracing the trajectory.

In our implementation, we set \(M=I\) since estimating a well-conditioned, geometry-adapted mass matrix requires additional tuning/curvature information \cite{hoffman14jmlr-nuts,betancourt17arxiv-hmc,girolami11jrss-rmhmc,wang13icml-adaptive}. We manually select \(\varepsilon\), and use the NUTS sampler \cite{hoffman14jmlr-nuts} to adaptively determine the effective number of integration steps \(L\) on the fly. Intuitively, this avoids hand-tuning \(L\): it extends the Hamiltonian trajectory until further integration would start to reverse direction in parameter space, yielding long proposals when beneficial while preventing redundant trajectories.

\textbf{Summary.} TSMC tackles the multi-modal, low-temperature target \eqref{eq:optimal-distribution} by introducing a tempering path that gradually deforms the easy-to-sample prior into the desired Boltzmann tilt. At each tempering level, we (i) reweight particles by importance sampling (using ESS to adaptively choose the next temperature), (ii) resample to control weight degeneracy, and (iii) apply HMC rejuvenation to restore diversity and explore the parameter space while preserving the tempered target distribution. Fig.~\ref{fig:tsmc-overview}(b) shows that TSMC successfully samples from the Shekel-tilted distribution: the final particles are well-distributed across the two global minima, with a small number of particles exploring the third local minimum. \warn{In \supp, we provide the asymptotic theoretical guarantees offered by TSMC.}
Next, we apply TSMC to trajectory optimization (TO) and policy optimization (PO).

%!TEX root = ../main.tex

\section{TSMC for Trajectory Optimization}
\label{sec:tsmc-to}

Recall from \eqref{eq:energy-function} the energy \(E(\theta)=\mathbb{E}_{x_0\sim\mu}\!\left[J(\tau(x_0,\theta))\right]\). For trajectory optimization (TO), \(\mu=\delta_{x_0}\), so \(E(\theta)=J(\tau(x_0,\theta))\) and \(\nabla_\theta E(\theta)\) can be computed exactly by differentiating through the trajectory rollout, using the adjoint method.

\begin{theorem}[Exact Gradients for TO]
\label{thm:gradient-computation-trajectory-optimization}
    Assume the dynamics \(f(x,u)\) is differentiable in $(x,u)$, the controller \(\pi_\theta(x)\) is differentiable in \((x,\theta)\), the stage cost functions \(\{ \ell_t \}_{t=0}^{T-1}\) are differentiable in \((x, u)\), and the terminal cost \(\ell_T\) is differentiable in \(x\). Starting from an initial state $x_0$, let $\tau=(x_0, u_0, \ldots, x_T)$ be a trajectory rollout generated by the controller $\pi_\theta(x)$. Define
    \begin{equation}\label{eq:gradient-dynamics-controller}
    \begin{aligned} 
    \hspace{-2mm} A_t \triangleq \frac{\partial f}{\partial x}(x_t,u_t) \in \Real{n\times n},  \ \ 
    B_t \triangleq \frac{\partial f}{\partial u}(x_t,u_t) \in \Real{n\times m}, \\
    \hspace{-2mm} L_t \triangleq \frac{\partial \pi_\theta}{\partial x}(x_t) \in \Real{m\times n}, \ \ 
    G_t \triangleq \frac{\partial \pi_\theta}{\partial \theta}(x_t) \in \Real{m\times d},\\
\ell_{x,t}\triangleq \frac{\partial \ell_t}{\partial x}(x_t,u_t) \in \Real{n}, \ \ 
\ell_{u,t}\triangleq \frac{\partial \ell_t}{\partial u}(x_t,u_t) \in \Real{m}, \\
\ell_{x,T}\triangleq \frac{\partial \ell_T}{\partial x}(x_T) \in \Real{n}
\end{aligned}
\end{equation}
and the costate (adjoint) recursion backward in time:
\begin{equation}\label{eq:adjoint-recursion}
\begin{aligned}
\varphi_T = \ell_{x,T}, \ \  \text{and for } t=T-1,\dots,0: \\
g_t = \ell_{u,t} + B_t^\top \varphi_{t+1}, \ \ 
\varphi_t = \ell_{x,t} + L_t^\top g_t + A_t^\top \varphi_{t+1}.
\end{aligned}
\end{equation}
Then the gradient of the trajectory cost with respect to \(\theta\) is
\begin{equation}\label{eq:grad-J}
\nabla_\theta J(\tau(x_0,\theta)) = \sum_{t=0}^{T-1} G_t^\top g_t.
\end{equation}
\end{theorem}
The result follows by repeated application of the chain rule (equivalently, reverse-mode automatic differentiation) and we omit the proof for brevity. In practice, when \(f\), \(\pi_\theta\), \(\{\ell_t\}_{t=0}^{T-1}\), and \(\ell_T\) are implemented in modern autodiff frameworks (\eg \textsc{PyTorch} or \textsc{JAX}), the adjoint recursion is computed efficiently by backpropagation through the differentiable rollout.

In the TO setting with open-loop parameterization \(\theta=(u_0,\ldots,u_{T-1})\in\Real{Tm}\) and \(\pi_\theta(x_t)=u_t\) (no state feedback), \(\frac{\partial \pi_\theta}{\partial x}(x_t)=0\), hence \(L_t=0\). Moreover, \(\frac{\partial \pi_\theta}{\partial \theta}(x_t)\) simply selects the \(t\)-th control block, \ie \(G_t=[0,\ldots,I_m,\ldots,0]\in\Real{m\times Tm}\), where the only nonzero block is the \(t\)-th \(m\times m\) identity.
%!TEX root = ../main.tex

\section{TSMC for Policy Optimization}
\label{sec:tsmc-po}

For policy optimization (PO), the energy \(E(\theta)\) is an expectation over random rollouts (through \(x_0\sim\mu\)), and is generally intractable to evaluate exactly under nonlinear dynamics. In practice, both \(E(\theta)\) and \(\nabla_\theta E(\theta)\) are approximated using Monte Carlo rollouts, yielding stochastic estimates. This raises a natural question: when only stochastic estimates of \(E(\theta)\) and \(\nabla_\theta E(\theta)\) are available, can we still use TSMC to sample from the optimal Boltzmann-tilted distribution \(p^\star\) in \eqref{eq:optimal-distribution}?

A first idea is to replace the deterministic HMC rejuvenation kernel from Section~\ref{sec:tsmc} with a stochastic-gradient variant (\eg stochastic gradient HMC \cite{chen14icml-sghmc}). While this can address the \emph{move} (MCMC) step, it does not resolve the \emph{weighting} step: the TSMC incremental weights \eqref{eq:tsmc-incremental-weight} depend on factors of the form \(\exp(-c\,E(\theta))\) (and ratios thereof) for a constant \(c>0\), and an unbiased estimator \(\widehat E(\theta)\) does \emph{not} in general yield an unbiased estimator of \(\exp(-c\,E(\theta))\). A principled route would be to construct a \emph{nonnegative unbiased estimator} \(\widehat W(\theta)\ge 0\) such that \(\mathbb{E}\big[\widehat W(\theta)\big]=\exp(-c\,E(\theta))\), and then plug \(\widehat W(\theta)\) into the incremental weights. Unfortunately, nonnegative unbiased estimation of nonlinear functionals such as \(\phi(z)=\exp(-z)\) is highly nontrivial and, in general, impossible without additional structure; see \cite{jacob15aos-nonnegative-unbiased}. Even when such estimators exist, they typically rely on specialized ``factory'' constructions and can require elaborate randomized procedures; see, \eg \cite{koskela25arxiv-debiased}.

Due to these challenges, we use two practical alternatives: (i) a deterministic approximation of \(E(\theta)\) and (ii) an extended-space TSMC variant.

\subsection{Deterministic Approximation}
\label{subsec:deterministic-approximation}

A simple way to recover a deterministic energy is to approximate the initial-state distribution \(\mu\) by an empirical measure supported on a \emph{fixed} set of \(B\) initial conditions \(\{x_0(b)\}_{b=1}^B\):
\begin{equation}\label{eq:muB}
\mu_B \;=\; \frac{1}{B}\sum_{b=1}^B \delta_{x_0(b)}.
\end{equation}
This yields the surrogate energy
\begin{equation}\label{eq:EB-po}
\hspace{-2mm} E_B(\theta):=\mathbb{E}_{x_0\sim\mu_B} \big[J(\tau(x_0,\theta))\big]
=\frac{1}{B}\sum_{b=1}^B J(\tau(x_0(b),\theta))
\end{equation}
and the corresponding Boltzmann-tilted target \(p^\star_B(\theta)\propto p_0(\theta)\exp(-E_B(\theta)/\lambda)\) (and similarly for the intermediate tempered targets along the TSMC path).
Crucially, once \(\{x_0(b)\}_{b=1}^B\) is fixed, both \(E_B(\theta)\) and \(\nabla_\theta E_B(\theta)\) are deterministic and can be computed by differentiating through each rollout using Theorem~\ref{thm:gradient-computation-trajectory-optimization} and averaging. Therefore, the TSMC algorithm from Section~\ref{sec:tsmc} does not need to be modified: incremental weights are evaluated using \(E_B(\theta)\), and the HMC rejuvenation kernel uses \(\nabla_\theta E_B(\theta)\) as usual.

% The main drawback is that the quality of this approximation depends on how well the fixed support points \(\{x_0(b)\}\) represent \(\mu\). In high-dimensional state spaces, selecting and fixing a representative set of initial conditions manually can be difficult, and too small or poorly chosen supports may lead to biased targets and brittle performance.

\subsection{Extended-space TSMC}
\label{subsec:extended-space-tsmc}

Pseudo-marginal methods address intractable expectations by augmenting the state space with auxiliary random variables, and then sampling the desired marginal by running a standard sampler on the extended space \cite{andrieu09aos-pseudomarginal,andrieu10jasa-pmcmc,chopin13jrssb-smc2,alenlov21jmlr-pmhmc}. Here we adopt this idea for PO by treating the rollout randomness (initial conditions) as part of the particle state.

\textbf{Extended-space Tempered Targets.}
Let each particle carry controller parameters \(\theta\in\Real{d}\) and a batch of \(B\) i.i.d.\ initial conditions \(x_0(1),\dots,x_0(B)\sim\mu\). We shorthand \(\mathbf{x}_0 \triangleq (x_0(1),\dots,x_0(B))\). Define the empirical average rollout cost
\begin{equation}\label{eq:pm-batch-cost}
\bar J_B(\theta, \mathbf{x}_0) \ \triangleq\ \frac{1}{B}\sum_{b=1}^B J(\tau(x_0(b),\theta)).
\end{equation}
We then define a sequence of tempered distributions, for $k=0,\dots,K$, on the extended space \((\theta,\mathbf{x}_0)\in\Real{d}\times(\Real{n})^{B}\):
\begin{equation}\label{eq:pm-tempered-path}
\tilde p_k(\theta,\mathbf{x}_0) \propto p_0(\theta)\prod_{b=1}^B \mu(x_0(b)) \exp\!\left(-\frac{\beta_k}{\lambda}\,\bar J_B(\theta,\mathbf{x}_0)\right).
\end{equation}
Running TSMC on \(\{\tilde p_k\}\) yields samples \((\theta,\mathbf{x}_0)\sim \tilde p_K\); discarding \(\mathbf{x}_0\) returns \(\theta\) distributed according to the marginal
\begin{equation}\label{eq:pm-marginal}
\tilde p_K(\theta) \propto p_0(\theta)\,\mathbb{E}_{\mathbf{x}_0\sim \mu^{\otimes B}}\!\left[\exp\!\left(-\frac{1}{\lambda}\,\bar J_B(\theta,\mathbf{x}_0)\right)\right]\!.
\end{equation}
This marginal differs from the original Boltzmann tilt \(p^\star(\theta)\propto p_0(\theta)\exp(-E(\theta)/\lambda)\): it is a \emph{risk-sensitive} (exponential utility) objective \cite{jacobson73tac-exponential}, reflecting that in general \(\mathbb{E}[\exp(-Z)]\neq \exp(-\mathbb{E}[Z])\). 
% As \(B\) increases, the batch average \(\bar J_B\) concentrates and the discrepancy between \eqref{eq:pm-marginal} and \(p^\star\) diminishes under standard concentration/regularity assumptions (\eg bounded or sub-Gaussian costs).
Indeed, write
\begin{equation}
\hspace{-6mm} \delta_B(\theta)
\coloneqq
\log \mathbb{E}_{\mathbf{x}_0\sim \mu^{\otimes B}}
\left[
\exp\left(-\frac{1}{\lambda}\bar J_B(\theta,\mathbf{x}_0)\right)
\right]
+\frac{1}{\lambda}E(\theta),\!\!\!\!\!\!\!
\end{equation}
we have, from~\eqref{eq:pm-marginal},
\begin{equation}
\tilde p_K(\theta)
\propto
p_0(\theta)\exp\left(-\frac{1}{\lambda}E(\theta)\right)
\exp(\delta_B(\theta)).
\end{equation}
Thus the extended-space marginal is a multiplicative
exponential-utility perturbation of the true Boltzmann tilt.
Moreover, if \(J(\tau(x_0,\theta))\) is sub-Gaussian under
\(x_0\sim\mu\) with variance proxy \(\sigma^2(\theta)\) \cite{wainwright19book-high}, then
\begin{equation}
0 \leq \delta_B(\theta)
\leq
\frac{\sigma^2(\theta)}{2B\lambda^2}.
\end{equation}
Consequently, the mismatch decreases as \(O(1/B)\), and is
small when the rollout-cost variance across initial states is
modest, the batch size \(B\) is large, and the temperature
\(\lambda\) is not too small.

\textbf{TSMC in the Extended Space.}
We maintain weighted particles \(\{(\theta_k(i),\mathbf{x}_{0,k}(i),w_k(i))\}_{i=1}^N\) approximating \(\tilde p_k\) as in \eqref{eq:pm-tempered-path}, where \(\mathbf{x}_{0,k}(i)\triangleq (x_{0,k}(i,1),\dots,x_{0,k}(i,B))\). The TSMC steps mirror those in Section~\ref{sec:tsmc}:
\begin{itemize}
    \item \emph{Initialize:} sample \(\theta_0(i)\sim p_0\) and \(x_{0,0}(i,b)\sim\mu\) i.i.d.\ for \(b=1,\dots,B\); set \(w_0(i)=1/N\).
    \item \emph{Importance weight update:} given particles for \(\tilde p_{k-1}\), update weights using the incremental ratio
    \begin{equation}\label{eq:pm-incremental-weight}
    \hspace{-5mm} \Delta w_k(i)\!=\!\exp\!\Big(-\!\frac{\beta_k-\beta_{k-1}}{\lambda}\!\,\bar J_B(\theta_{k-1}(i),\mathbf{x}_{0,k-1}(i))\!\Big),\!\!\!
    \end{equation}
    and normalize as in \eqref{eq:tsmc-weight-update}.
    \item \emph{Resample:} resample ancestor indices according to \(w_k(i)\) and reset weights to \(1/N\), as in Section~\ref{sec:tsmc}.
    \item \emph{HMC rejuvenation (move \(\theta\), hold \(\mathbf{x}_0\) fixed):} for each particle, apply an HMC move targeting the conditional density
    \(\tilde p_k(\theta\mid \mathbf{x}_0)\propto p_0(\theta)\exp\!\big(-\frac{\beta_k}{\lambda}\bar J_B(\theta,\mathbf{x}_0)\big)\).
    Define the unnormalized negative log-density (potential)
    \begin{equation}\label{eq:pm-potential}
    \tilde V_k(\theta)\ \triangleq\ \frac{\beta_k}{\lambda}\bar J_B(\theta,\mathbf{x}_0)\ -\ \log p_0(\theta),
    \end{equation}
    and Hamiltonian $\tilde H_k(\theta,r)\triangleq\tilde V_k(\theta) + 0.5r^\top M^{-1}r$,
    where \(r\sim\mathcal{N}(0,M)\). The gradient of the potential is
    \begin{equation}\label{eq:pm-potential-grad}
    \hspace{-2mm} \nabla_\theta \tilde V_k(\theta) \!=\! \frac{\beta_k}{\lambda B}\sum_{b=1}^B \nabla_\theta J(\tau(x_0(b),\theta)) \!-\! \nabla_\theta \log p_0(\theta),\!\!
    \end{equation}
    where each \(\nabla_\theta J(\tau(x_0(b),\theta))\) is obtained by Theorem~\ref{thm:gradient-computation-trajectory-optimization}.
    One HMC move then proceeds by (i) sampling \(r_0\sim\mathcal{N}(0,M)\), (ii) applying the leapfrog integrator \eqref{eq:hmc-leapfrog} with \(V_k\) replaced by \(\tilde V_k\) to obtain \((\theta_L,r_L)\), and (iii) accepting \(\theta'=\theta_L\) with probability \eqref{eq:hmc-accept} where \(H_k\) is replaced by \(\tilde H_k\) or otherwise keeping \(\theta'=\theta_0\).
    \item \emph{Refresh initial states (move \(\mathbf{x}_0\), hold \(\theta\) fixed):} to avoid restricting \(\mathbf{x}_0\) to a finite set of initial samples, we update the auxiliary batch using a Metropolis--Hastings refresh that targets \(\tilde p_k(\mathbf{x}_0\mid\theta)\). For each \(b=1,\dots,B\), propose \(x_0'(b)\sim \mu(\cdot)\) and accept with probability
    \begin{equation}\label{eq:pm-x0-refresh-accept}
    \hspace{-2mm}\min\!\left\{ \!1,\exp\!\left(-\frac{J(\tau(x_0'(b),\theta)) \!-\! J(\tau(x_0(b),\theta)) }{\lambda B / \beta_k}\right) \! \right\}.\!\!\!\!\!
    \end{equation}
    Otherwise, keep \(x_0(b)\).
\end{itemize}

The rejuvenation kernel at level $k$ is defined as the composition of (i) an HMC move on \(\theta\) targeting \(\tilde p_k(\theta\mid \mathbf{x}_0)\) with \(\mathbf{x}_0\) held fixed, and (ii) an MH refresh of \(\mathbf{x}_0\) targeting \(\tilde p_k(\mathbf{x}_0\mid \theta)\) with \(\theta\) held fixed. Each sub-step leaves its corresponding conditional invariant, hence their composition leaves the joint extended target \(\tilde p_k(\theta,\mathbf{x}_0)\) invariant.

\begin{remark}[Extension to Stochastic Dynamics] \label{remark:stochastic-dynamics}
    A natural extension of our framework is to stochastic dynamics of the form \(x_{t+1}=f(x_t,u_t)+w_t\), where \(w_t\) is a random disturbance. In that case, the energy becomes \(E(\theta)=\mathbb{E}_{x_0,w_{0:T-1}}[J(\tau(x_0,\theta;w_{0:T-1}))]\). Thus, at the level of the target distribution, TSMC carries over directly by {redefining the energy to average over both initial-state and disturbance randomness}. Moreover, the extended-space construction also generalizes naturally by augmenting the auxiliary variables with sampled disturbance trajectories. The resulting \(\theta\)-marginal is again a risk-sensitive surrogate of the exact Boltzmann target. 
\end{remark}
% \red{Maybe add Pseudocode.}

%!TEX root = ../main.tex

\section{Experiments: Trajectory Optimization}
\label{sec:experiments-to}

\begin{figure*}[h]
    \centering
    \includegraphics[width=\textwidth]{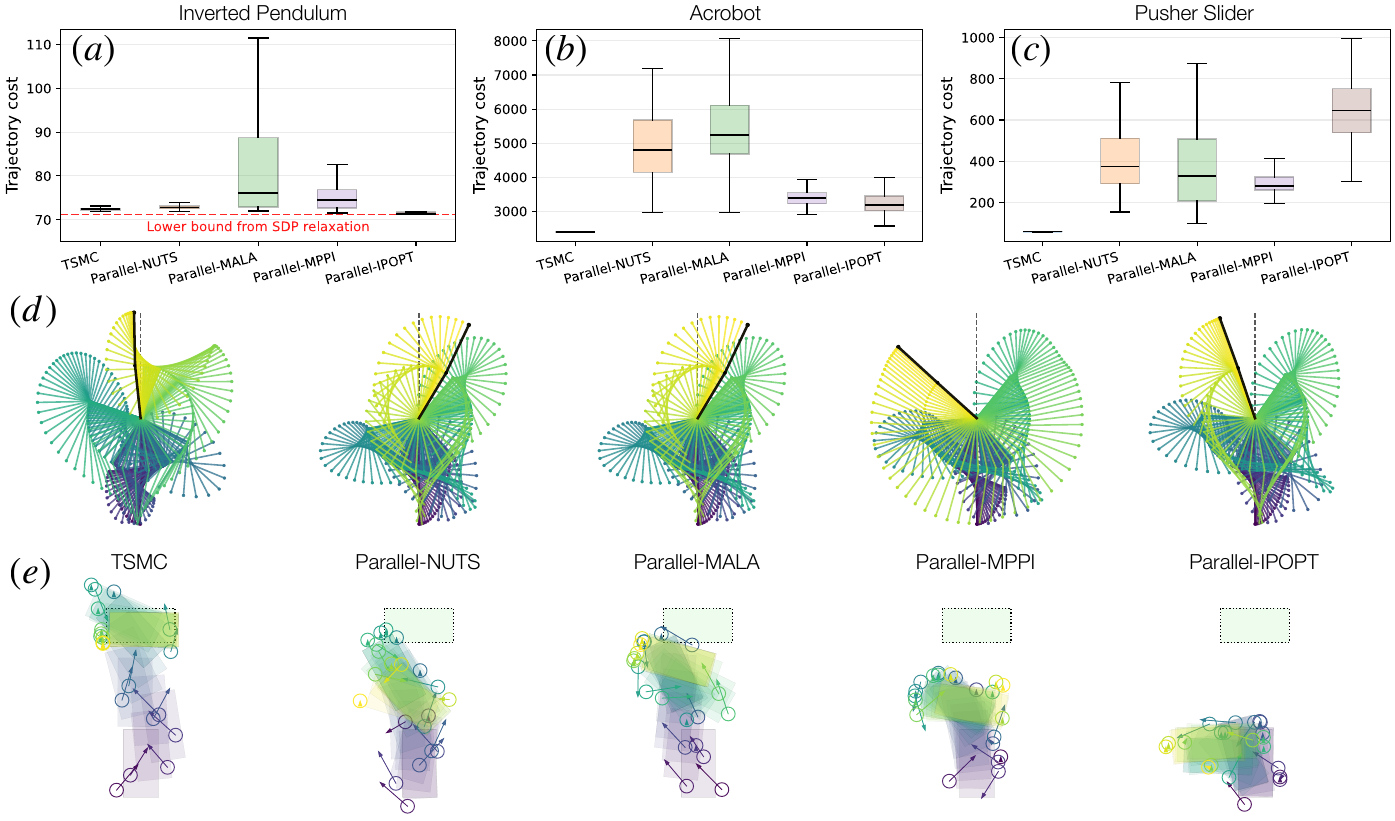}
    \vspace{-8mm}
    \caption{Trajectory optimization results. (a)--(c): boxplots of final-particle trajectory costs for \tsmc{} and baselines. (d)--(e): best-cost trajectories for Acrobot and pusher--slider, respectively. The colormap indicates progression from initial to final states.}
    \vspace{-4mm}
    \label{fig:to-figure}
\end{figure*}

We evaluate TSMC for trajectory optimization against baseline methods. All experiments run on a Lambda workstation with an AMD Ryzen Threadripper PRO 5975WX (32 cores) and two NVIDIA Ada6000 GPUs.

\textbf{Implementation.} We implement the dynamical systems and TSMC in \jax~\cite{bradbury18software-jax} to leverage automatic differentiation and GPU acceleration. We use \blackjax~\cite{cabezas24arxiv-blackjax} to implement TSMC; specifically, its adaptive TSMC framework with NUTS as the HMC rejuvenation kernel. We set the ESS threshold for selecting the next tempering parameter to $\rho=0.8$. 

% The dynamics \(f(\cdot,\cdot)\) and cost functions \(\{\ell_t\}\) are implemented in \jax\xspace and differentiated automatically within the TSMC pipeline. System details are provided in the subsections below.

\textbf{Baselines.} We compare \tsmc with:
\begin{itemize}
    \item \pnuts. \tsmc starts from particles sampled from \(p_0\). \pnuts uses the same initialization, but for each particle runs an independent MCMC using NUTS in \blackjax, targeting the optimal distribution \(p^\star\) directly. 
    
    % NUTS hyperparameters match those used in \tsmc.
    \item \pmala. Same as \pnuts, but using the Metropolis-adjusted Langevin algorithm (MALA) \cite{besag94jrssb-mala} in \blackjax. Since MALA is cheaper per iteration than NUTS, we run MALA for \(100\times\) more steps.
    \item \pmppi. Starting from the same initial particles as \tsmc, we run \(64\) MPPI planning steps per particle. We treat each particle as the nominal control, sample 64 noise perturbations, evaluate their trajectory costs, and update by an \(\exp(-\frac{1}{\lambda}\text{cost})\)-weighted average.
    \item \pipopt. Starting from the same initial particles, we run IPOPT using \casadi~\cite{andersson19mpc-casadi} with single shooting, initialized at each particle. We keep the \casadi~dynamics consistent with the \jax~implementation so the trajectory gradients with respect to controls match. Since IPOPT can be slow, we run it in parallel across CPU cores (25 workers). We set the maximum number of iterations to 200 and the stopping tolerance to \(10^{-6}\).
\end{itemize}
Comparing \tsmc with \pnuts and \pmala isolates the role of tempering for sampling from sharp, multimodal targets. Comparing with \pmppi highlights the benefit of combining importance sampling with gradient-based optimization. Finally, \pipopt serves as a simple hybrid baseline: gradient-based optimization from multiple random initializations. Across all methods, we use \(N=100\) particles.

\textbf{Prior.} We use a first-order autoregressive prior \cite{hamilton20book-time} over the open-loop control sequence to encourage smoothness: \(u_t = \gamma u_{t-1} + \epsilon_t\), where \(\epsilon_t \sim \mathcal{N}(0,\sigma^2 I)\), with \(\gamma=0.9\) and \(\sigma=0.3\).

% Next, we describe the dynamics of three systems and the experimental results.

\subsection{Inverted Pendulum}
\label{subsec:experiments-to-inverted-pendulum}

% We first consider the inverted pendulum problem with control bounds. This is a low-dimensional task for which good solutions are relatively easy to find, and we therefore expect all methods to perform well.

\textbf{Setup.} The goal is to swing up the pendulum to the upright configuration under control limits. We discretize the dynamics using a variational integrator~\cite{lee08thesis-computational,kang24wafr-strom,teng23rss-convex}. Thanks to its favorable energy behavior, this discretization permits a relatively large time step and a shorter horizon. We follow \cite[Section E.1]{kang24wafr-strom}; in particular, we use a time step $\Delta t = 0.1$ seconds and planning horizon $T = 30$. Since the control is one-dimensional, this yields $\theta \in \Real{30}$.

\textbf{Certified Lower Bound.} We use the semidefinite programming (SDP) relaxations in \cite{kang24wafr-strom,kang25rss-spot} to compute a certified lower bound on the optimal cost, providing a reference for assessing whether \tsmc{} (and baselines) reach global optimality.

\textbf{Results.} Fig.~\ref{fig:to-figure}(a) shows boxplots of the trajectory costs across the final particles for each method, along with the certified lower bound. We make three observations. (i) \tsmc, \pnuts, and \pipopt achieve similarly strong performance: their final particles attain costs close to the SDP lower bound, suggesting near-global solutions. (ii) \pmala performs reasonably, but exhibits a heavier tail: while the best cost is near the lower bound, some particles remain far from optimality. (iii) \pmppi outperforms \pmala but underperforms \tsmc, \pnuts, and \pipopt, highlighting the benefit of combining importance sampling with gradient-based optimization under differentiable dynamics. 

% Since the best-cost trajectories are near-optimal and visually similar across methods, \red{we include them in the \supp~for completeness.}

\subsection{Acrobot}
\label{subsec:experiments-to-acrobot}

% We next consider the acrobot problem with control bounds. This task is more challenging, and we will see a clearer separation between methods.

\textbf{Setup.} The goal is to swing up the acrobot (two-link pendulum) to the upright configuration using a single control at the elbow joint (Fig.~\ref{fig:to-figure}(d)).
% \red{The continuous-time dynamics of the acrobot system, its parameters, and the stage-wise and terminal cost functions are provided in \supp}.
We discretize the continuous-time dynamics using fourth-order Runge--Kutta (RK4) with time step $\Delta t = 0.025$ seconds and planning horizon $T = 200$. Since the control is one-dimensional, this yields $\theta \in \Real{200}$.

\textbf{Results.} Fig.~\ref{fig:to-figure}(b) shows boxplots of the trajectory costs across the final particles for each method. Here, \tsmc{} significantly outperforms all baselines. Fig.~\ref{fig:to-figure}(d) visualizes the best-cost trajectories from each method (final state highlighted in thick black): \tsmc{} is the only method that achieves a successful swing-up close to the upright configuration. Numerically, the best trajectory cost from \tsmc{} is $2389$, while the best costs from the baselines are $2967$ for \pnuts, $2967$ for \pmala, $2694$ for \pmppi, and $2562$ for \pipopt.

\subsection{Pusher Slider}
\label{subsec:experiments-to-pusher-slider}

\textbf{Setup.} We consider the planar pusher--slider problem, which features nonsmooth contact dynamics \cite{mason86ijrr-mechanics,lynch96ijrr-stable,hogan20wafr-feedback}. Inspired by \cite{kang25rss-spot}, we model the pusher--slider dynamics only during contact between the pusher (a circular disc) and the slider (a rectangular box), see Fig.~\ref{fig:to-figure}(e); any free-space motion of the pusher is assumed to be planned separately. The state is the planar pose of the slider in world coordinates, \(x=[b_x,b_y,\alpha]\in\mathbb{R}^3\), where \((b_x,b_y)\) is the box center and \(\alpha\) is the box orientation. We use a 4D \emph{contact-mode} control \(u=[e,s,v_x,v_y]\in\mathbb{R}^4\): \(e\in\{1,2,3,4\}\) selects the contacted edge, \(s\in[-s_{\max},s_{\max}]\) selects the contact location along that edge (with \(s=\pm 1\) corresponding to corners and \(s_{\max}\in(0,1]\) as a safety margin), and \(v_x,v_y\) are pusher velocity commands expressed in the slider frame and bounded componentwise by \(|v_x|,|v_y|\le u_{\max}\). Given \((e,s,v_x,v_y)\), we use the quasi-static model in \cite{lynch96ijrr-stable} to compute the slider twist and update the slider pose. This model is differentiable except for the discrete edge variable \(e\). To obtain an end-to-end differentiable pipeline, we relax \(e\) by introducing a continuous variable \(\hat e\in[1,4]\), computing its distances to the integers in \(\{1,2,3,4\}\), mapping these distances to a categorical distribution via a softmax, and then using the Gumbel--Softmax trick \cite{jang17iclr-categorical} to sample \(e\) through a differentiable relaxation. The trajectory cost includes per-stage pose error to the goal pose and a terminal pose error, together with a temporal-smoothness penalty that discourages frequent contact-mode switches. We use a planning horizon \(T=200\), which yields \(\theta\in\Real{800}\). For \casadi{}+IPOPT, we cannot use Gumbel--Softmax; instead, we use a soft mixture model that forms a contact normal direction as a weighted combination of per-edge normals using the softmax weights. Since this mixture is not physically realistic, after IPOPT we round the edge variables to the nearest integers and forward simulate to obtain physically valid trajectories for cost evaluation.

\textbf{Results.} Fig.~\ref{fig:to-figure}(c) shows boxplots of trajectory costs across the final particles for each method. Here, \tsmc{} significantly outperforms all baselines and is the only method that reaches the goal pose. Fig.~\ref{fig:to-figure}(e) visualizes the best-cost trajectories.

\section{Experiments: Policy Optimization}
\label{sec:experiments-po}

\begin{figure*}[t]
    \centering
    \includegraphics[width=\textwidth]{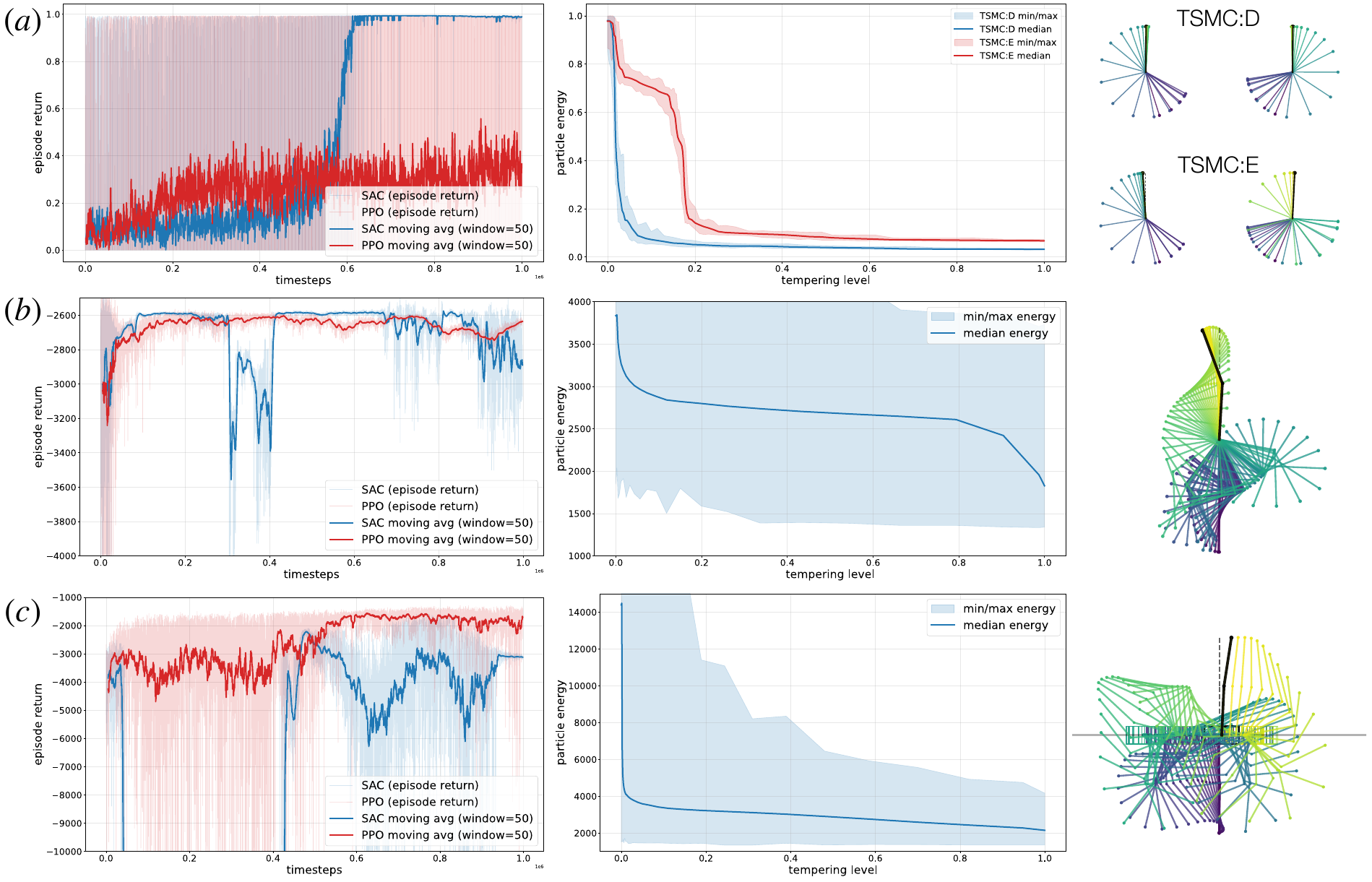}
    \vspace{-8mm}
    \caption{Policy optimization results. Left to right: episode returns of \ppo{} and \sac{} \wrt interaction steps; distribution of particle energies \wrt tempering levels in \tsmc; trajectory rollouts of the policy learned by \tsmc. (a) Inverted pendulum with sparse rewards; (b) Acrobot; (c) Double pendulum on a cart.}
    \vspace{-5mm}
    \label{fig:po-figure}
\end{figure*}

We next turn to policy optimization benchmarks. We initially aimed to evaluate on MuJoCo locomotion tasks (\eg HalfCheetah and Ant), which involve contact dynamics. While both MuJoCo JAX (MJX) \cite{todorov12iros-mujoco} and Brax \cite{freeman21arxiv-brax} provide JAX-based implementations, we found MJX's contact solver relies on a \texttt{while} loop that prevents reverse-mode differentiation, as documented in GitHub issues \cite{mujoco24issue-2259,mujoco23issue-1182}. Brax supports differentiable contact, but for HalfCheetah we observed gradient norms as large as \(10^9\) to \(10^{13}\), which makes it unsuitable for our framework. We therefore benchmark on classical control tasks without contact---where we can already compare against strong RL baselines---and leave contact-rich tasks to future work as differentiable contact solvers in \jax{} mature.

\textbf{Implementation.} We use the same implementation details as in Section~\ref{sec:experiments-to}, except that for the extended-space TSMC variant in Section~\ref{subsec:extended-space-tsmc} we implement adaptive tempering from scratch, since \blackjax's built-in adaptive TSMC routine does not support alternating two block moves over \(\theta\) and \(\mathbf{x}_0\).

\textbf{Policy Architecture.} We parameterize \(\pi_\theta\) as a multi-layer perceptron (MLP) with two hidden layers of 32 units each.

\textbf{Baselines.} We compare against two policy optimization algorithms for RL: proximal policy optimization (\ppo) \cite{schulman17arxiv-ppo} and soft actor-critic (\sac) \cite{haarnoja18icml-sac}. We wrap each dynamical system as a Gymnasium environment \cite{towers25neurips-gymnasium} and run \ppo{} and \sac{} using Stable-Baselines3 \cite{raffin21jmlr-sb3} with default hyperparameters, $10^6$ timesteps, and 32 parallel environments.

\subsection{Inverted Pendulum with Sparse Rewards}
\label{subsec:experiments-po-inverted-pendulum-sparse-rewards}

\textbf{Setup.} Following Section~\ref{subsec:experiments-to-inverted-pendulum}, we use a time step $\Delta t = 0.1$ seconds and planning horizon $T = 32$. To make the task challenging, we use a terminal cost only formulation. Let $(x_{T,\position}, x_{T,\velocity}) \in \Real{4}$ denote the pendulum's position and velocity at time $T$, and let $(x_{g,\position}, x_{g,\velocity})$ be the goal state. Define the goal distance
$\distance = \sqrt{\Vert x_{T,\position} - x_{g,\position}\Vert^2} + 0.5 \sqrt{\Vert x_{T,\velocity} - x_{g,\velocity}\Vert^2}$,
and set the terminal cost to $\ell_T = 1 - \sigma\left( (\varepsilon - \distance)/\epsilon \right)$,
where $\sigma(\cdot)$ is the sigmoid function, $\varepsilon = 0.1$, and $\epsilon = 0.02$ (the sparse reward is $1-\ell_T$). The policy parameters satisfy $\theta \in \Real{1184}$. We set the initial-state distribution to be uniform over the 2D box $[-\pi, \pi] \times [-\pi, \pi]$. For both \tsmcd{} (deterministic approximation) and \tsmce{} (extended-space), we use $B=3000$ and $N=32$.

\textbf{Results.} Fig.~\ref{fig:po-figure}(a) (left) shows learning curves for \ppo{} and \sac{}. \ppo{} does not reliably improve within the allotted budget, whereas \sac{} attains strong returns, consistent with the benefits of off-policy learning and experience replay. Fig.~\ref{fig:po-figure}(a) (middle) shows the distribution of particle energies across tempering levels in \tsmc. Both variants achieve similar performance as \sac{}. Fig.~\ref{fig:po-figure}(a) (right) visualizes representative rollouts. 

% We note that \ppo{} and \sac{} maximize rewards, whereas \tsmc{} minimizes costs.

\subsection{Acrobot}
\label{subsec:experiments-po-acrobot}

\textbf{Setup.} We follow Section~\ref{subsec:experiments-to-acrobot} and use a time step $\Delta t = 0.04$ seconds with planning horizon $T = 100$, which yields $\theta \in \Real{1184}$. We use a single-state initial distribution, \ie $\mu = \delta_{x_0}$, with $x_0$ corresponding to the bottom-right configuration. We focus on a single initial state because Acrobot is substantially more challenging; as we show below, this setting is already difficult for both \ppo{} and \sac{}. Under $\mu=\delta_{x_0}$, the \tsmc{} variants coincide; we use a large particle count $N=16000$ for \tsmc{} (smaller particle sets did not yield a successful policy).

\textbf{Results.} Fig.~\ref{fig:po-figure}(b) (left) shows learning curves for \ppo{} and \sac{}. Neither method reliably reaches a successful swing-up. 
% One possible explanation is that near-optimal behavior is close to bang-bang control; in our differentiable rollout formulation, this can induce large sensitivity terms $G_t$ in \eqref{eq:gradient-dynamics-controller} and hence high-variance gradient estimates. 
Fig.~\ref{fig:po-figure}(b) (middle) shows the distribution of particle energies across tempering levels in \tsmc. 
% The distribution remains relatively spread, reflecting the difficulty of the landscape; nonetheless, 
The final-step best-cost particle attains a trajectory cost of $1338.15$, which improves over the average trajectory cost over the final 50 episodes of $2311.14$ for \ppo{} and $2482.44$ for \sac{}. Fig.~\ref{fig:po-figure}(b) (right) visualizes the resulting swing-up trajectory from \tsmc. To investigate whether the performance gap is due to limited training budget, we run \ppo{} and \sac{} for $10^9$ environment steps. Fig.~\ref{fig:acrobot_ppo_vs_sac_learning_curve_long} shows the learning curves. Neither method reaches a successful swing-up.

\begin{figure}
    \centering
    \includegraphics[width=\linewidth]{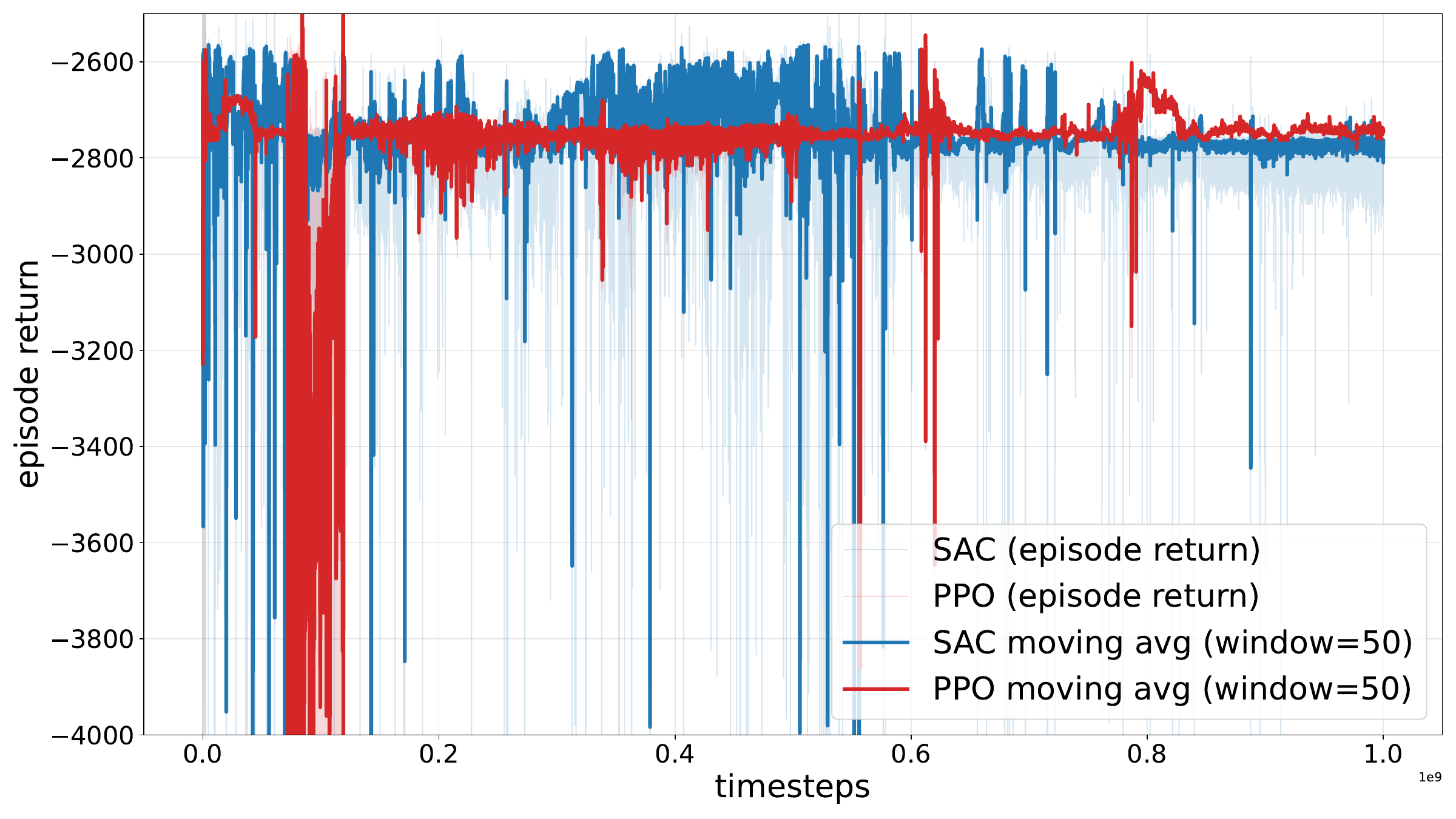}
    \vspace{-8mm}
    \caption{Running \ppo{} and \sac{} on Acrobot for $10^9$ environment steps.
    \label{fig:acrobot_ppo_vs_sac_learning_curve_long}}
    \vspace{-4mm}
\end{figure}

\subsection{Double Pendulum on a Cart}

\textbf{Setup.} We consider swinging up a two-link double pendulum mounted on a cart to the upright configuration \cite{graichen07automatica-swing}. This task is similar to Acrobot but has a higher-dimensional state, $x \in \Real{6}$. We discretize the continuous-time dynamics using RK4 with a time step $\Delta t = 0.06$ seconds and planning horizon $T = 100$, yielding $\theta \in \Real{1248}$. We use a single-state initial distribution, with $x_0$ corresponding to the bottom-right configuration. We use $N=14000$ particles for \tsmc.

\textbf{Results.} Fig.~\ref{fig:po-figure}(c) (left) shows learning curves for \ppo{} and \sac{}. The average trajectory costs over the final 50 episodes are $1713.00$ for \ppo{} and $3121.77$ for \sac{}. Fig.~\ref{fig:po-figure}(c) (middle) shows particle energy distributions across tempering levels in \tsmc; the best cost among the final particles is $1350.50$, improving over both \ppo{} and \sac{}. Fig.~\ref{fig:po-figure}(c) (right) visualizes the resulting swing-up trajectory from \tsmc.

%!TEX root = ../main.tex

\section{Conclusion}
\label{sec:conclusion}

We introduced TSMC, a sampling-based framework for trajectory and policy optimization that exploits differentiable dynamics, and demonstrated its robustness and versatility across TO and PO benchmarks.

\textbf{Limitations.} While TSMC is naturally parallelizable across particles, it can be computationally and memory intensive: differentiating through rollouts can be slow, and maintaining many particles can be demanding on GPU memory, potentially requiring multi-GPU training.

\textbf{Future work.} (i) Integrate TSMC with differentiable contact solvers and evaluate it on training RL policies for contact-rich tasks. (ii) Accelerate TSMC via customized low-level GPU kernels and hardware acceleration.
% (ii) Extend TSMC to online planning with learned dynamics and a pretrained prior policy $p_0$. 

\section*{Acknowledgments}
We thank Sitan Chen for discussions on sampling from the Boltzmann-tilted distribution, and Shucheng Kang and Haoyu Han for helpful conversations throughout the project and for proofreading early drafts. We acknowledge funding from the Office of Naval Research grant N00014-25-1-2322.

\clearpage

%%%%%%%%%%%%%%%%%%%%%%%%%%%
%%% Supplementary
%%%%%%%%%%%%%%%%%%%%%%%%%%%

\onecolumn

\renewcommand{\thesection}{S-\Roman{section}}
\renewcommand{\thesubsection}{\thesection.\Alph{subsection}}
\renewcommand{\thesubsubsection}{\thesubsection.\arabic{subsubsection}}
% Tables: Table S-I, S-II, ...
\renewcommand{\thetable}{S-\Roman{table}}
% Figures: Fig. S1, S2, ...
\renewcommand{\thefigure}{S\arabic{figure}}
% Equations: (S1), (S2), ...
\renewcommand{\theequation}{S\arabic{equation}}
% Theorems/remarks: Theorem S1, Remark S2, ...
\renewcommand{\thetheorem}{S\arabic{theorem}}

\begin{center}
   {\Large \bf Supplementary Material}
\end{center}

\textbf{Outline.} We provide a self-contained proof to Theorem~\ref{thm:optimal-distribution} in Section~\ref{sec:proof-optimal-distribution}. We provide theoretical asymptotic consistency results for TSMC in Section~\ref{sec:theoretical-guarantees}. We give supplementary results and discussion on the trajectory optimization and policy optimization experiments in Section~\ref{sec:supplementary-experiments}. We conclude with related work in Section~\ref{sec:related-work}.

%!TEX root = ../main.tex

\section{Proof of Theorem \ref{thm:optimal-distribution}}
\label{sec:proof-optimal-distribution}
\begin{proof}
Recall the definition of the energy function $E(\theta)=\mathbb{E}_{x_0\sim\mu}[J(\tau(x_0,\theta))]$ as in \eqref{eq:energy-function}. We can write the KL-regularized objective in \eqref{eq:regularized-trajectory-policy-optimization} as
\[
\mathcal{F}(p) = \int_{\Real{d}} p(\theta)E(\theta)\,d\theta + \lambda \int_{\Real{d}} p(\theta)\log\frac{p(\theta)}{p_0(\theta)}\,d\theta.
\]
The feasible set of distributions is given by
\[
\mathcal{P} = \Big\{p\ge 0:\int_{\Real{d}} p(\theta)\,d\theta = 1\Big\}.
\]

\textbf{Strict Convexity.} $\mathcal{F}(p)$ is strictly convex in $p$, due to the first term being linear and the second KL term being strictly convex. The feasible set $\mathcal{P}$ is clearly a convex set. Therefore, if a minimizer exists, it is unique.

\textbf{Optimality Condition.} Due to convexity, we can use the method of Lagrange multipliers to find the minimizer.
Introduce a scalar Lagrange multiplier $\alpha\in\mathbb{R}$ for the constraint $\int p=1$ and consider the Lagrangian functional
\[
\mathcal{L}[p,\alpha]
:= \int \Big(p(\theta)E(\theta)+\lambda p(\theta)\log\frac{p(\theta)}{p_0(\theta)}\Big)\,d\theta
+\alpha\Big(\int p(\theta)\,d\theta-1\Big).
\]
Let $p$ be an interior point of $\mathcal{P}$ (so $p(\theta)>0$ a.e.) and let $\delta p$ be an arbitrary perturbation
with $\int \delta p = 0$ and such that $p+\varepsilon\delta p\in\mathcal{P}$ for small $\varepsilon$.
Compute the first variation:
\[
\left.\frac{d}{d\varepsilon}\right|_{\varepsilon=0}\int (p+\varepsilon\delta p)E\,d\theta
= \int \delta p(\theta)E(\theta)\,d\theta,
\]
and using $\frac{d}{dt}(t\log t)=\log t+1$ for $t>0$,
\[
\left.\frac{d}{d\varepsilon}\right|_{\varepsilon=0}\int (p+\varepsilon\delta p)\log(p+\varepsilon\delta p)\,d\theta
= \int \delta p(\theta)\big(\log p(\theta)+1\big)\,d\theta,
\]
while
\[
\left.\frac{d}{d\varepsilon}\right|_{\varepsilon=0}\int (p+\varepsilon\delta p)\log p_0\,d\theta
= \int \delta p(\theta)\log p_0(\theta)\,d\theta,
\qquad
\left.\frac{d}{d\varepsilon}\right|_{\varepsilon=0}\alpha\int (p+\varepsilon\delta p)\,d\theta
= \alpha\int \delta p\,d\theta.
\]
Therefore
\begin{equation}\label{eq:delta-mathcal-L}
\delta \mathcal{L}
= \int \delta p(\theta)\Big[ E(\theta) + \lambda(\log p(\theta)+1-\log p_0(\theta)) + \alpha\Big]\,d\theta.
\end{equation}
A stationary point must satisfy $\delta\mathcal{L}=0$ for all admissible $\delta p$,
which forces the bracketed term in \eqref{eq:delta-mathcal-L} to vanish a.e.:
\begin{equation}\label{eq:stationary}
E(\theta) + \lambda\big(\log p(\theta)+1-\log p_0(\theta)\big) + \alpha = 0 \quad\text{a.e.}
\end{equation}
Rearranging \eqref{eq:stationary} gives
\[
\log p(\theta) = \log p_0(\theta) - \frac{1}{\lambda}E(\theta) - 1 - \frac{\alpha}{\lambda}.
\]
Exponentiating yields
\[
p(\theta) = C\, p_0(\theta)\exp\!\Big(-\tfrac{1}{\lambda}E(\theta)\Big),
\qquad
C:=\exp\!\Big(-1-\tfrac{\alpha}{\lambda}\Big)>0.
\]
Imposing $\int p=1$ gives
\[
1 = C\int p_0(\theta)\exp\!\Big(-\tfrac{1}{\lambda}E(\theta)\Big)\,d\theta = C Z,
\]
so $C=1/Z$ and the candidate optimizer is exactly \eqref{eq:optimal-distribution}.

\textbf{Well-definedness under a Lower Bound on $J$.}
Assume $J(\cdot)\ge - M$ is lower bounded for some constant $M > 0$. It follows that $E(\theta)\ge - M$ for all $\theta$. 
Then for all $\theta$,
\[
\exp\!\Big(-\tfrac{1}{\lambda}E(\theta)\Big) \;\le\; \exp\!\Big(\tfrac{M}{\lambda}\Big),
\]
hence
\[
Z = \int p_0(\theta)\exp\!\Big(-\tfrac{1}{\lambda}E(\theta)\Big)\,d\theta
\;\le\; e^{M/\lambda}\int p_0(\theta)\,d\theta
= e^{M/\lambda} < \infty.
\]
Moreover, since the integrand is nonnegative and $p_0>0$ a.e., we have $Z>0$.
Thus $p^\star$ in \eqref{eq:optimal-distribution} is a well-defined probability density. Note that in the main Theorem, we assumed that $J(\cdot)\ge 0$. Here, we have shown that it can be relaxed to $J(\cdot)\ge - M$ for some constant $M > 0$.

In summary, we have shown that the optimal distribution $p^\star$ is given by \eqref{eq:optimal-distribution} and is a well-defined probability density. Because $\mathcal{F}$ is strictly convex on the convex feasible set $\mathcal{P}$,
any feasible stationary point is the unique global minimizer (a.e.).
Since $p^\star$ is feasible and satisfies the first-order condition derived above, it is the unique minimizer.
\end{proof}
    
%!TEX root = ../main.tex

\section{Theoretical Guarantees}
\label{sec:theoretical-guarantees}

In Section \ref{sec:tsmc}, we introduced the tempered sequential Monte Carlo (TSMC) algorithm for sampling from the Boltzmann-tilted distribution \eqref{eq:optimal-distribution}. The interested reader may be curious about the theoretical guarantees of TSMC. In this section, we provide a standard consistency result for TSMC.

\begin{theorem}[Consistency of TSMC (law of large numbers)]
    \label{thm:tsmc_consistency}
    Let $p_0$ be a probability density on $\Real{d}$.
    Fix a tempering schedule $0=\beta_0<\beta_1<\cdots<\beta_K=1$ and a measurable energy $E:\Real{d}\to\mathbb{R}$.
    For each $k=0,1,\dots,K$, define the intermediate target
    \begin{equation}
    p_k(\theta)\;:=\;\frac{1}{Z_k}\,p_0(\theta)\exp\!\Big(-\frac{\beta_k}{\lambda}E(\theta)\Big),
    \qquad
    Z_k:=\int_{\Real{d}} p_0(\theta)\exp\!\Big(-\frac{\beta_k}{\lambda}E(\theta)\Big)\,d\theta,
    \end{equation}
    and assume $Z_k\in(0,\infty)$ for all $k$ (\eg it suffices that $E$ is lower bounded).
    
    Consider the TSMC algorithm producing particles $\{\theta_k (i)\}_{i=1}^N$ at each level $k$ via:
    (i) \emph{reweighting} by the incremental weight function
    \begin{equation}
    G_k(\theta)\;:=\;\frac{p_k(\theta)}{p_{k-1}(\theta)}
    \;\propto\;
    \exp\!\Big(-\frac{\beta_k-\beta_{k-1}}{\lambda}E(\theta)\Big),
    \qquad k=1,\dots,K,
    \end{equation}
    (ii) \emph{resampling} using an \emph{unbiased} resampling scheme (\eg using typical resampling schemes such as systematic resampling offered by \blackjax),
    and (iii) an optional \emph{move} step using any Markov kernel $\calT_k(\cdot\,|\,\cdot)$
    that leaves $p_k$ invariant.
    No mixing/ergodicity assumption on $\calT_k$ is required (the identity kernel is allowed).
    
    Assume furthermore that for each $k=1,\dots,K$ there exists $q>1$ such that
    \begin{equation}\label{eq:moment-assumption}
    \mathbb{E}_{\theta\sim p_{k-1}}\!\left[G_k(\theta)^q\right] < \infty.
    \end{equation}
    Let the empirical measure at level $k$ be
    \[
    \eta_k^N \;:=\; \frac{1}{N}\sum_{i=1}^N \delta_{\theta_k (i) }.
    \]
    Then for every bounded measurable test function $\varphi:\Real{d}\to\mathbb{R}$ and every $k=0,1,\dots,K$,
    \begin{equation}
    \eta_k^N(\varphi)\;:=\;\frac{1}{N}\sum_{i=1}^N \varphi(\theta_k (i) )
    \;\xrightarrow[N\to\infty]{\mathbb{P}}\;
    p_k(\varphi)\;:=\;\int_{\Real{d}}\varphi(\theta)\,p_k(\theta)\,d\theta.
    \end{equation}
    In particular, $\eta_K^N$ converges weakly in probability to $p_K$.
    \end{theorem}

    \begin{remark}[Sufficient Condition for the Moment Assumption]
    If $E$ is bounded below, \ie $E(\theta)\ge -M$ for some $M > 0$, then $G_k(\theta)\le \exp((\beta_k-\beta_{k-1})M/\lambda)$,
    so $\mathbb{E}_{p_{k-1}}[G_k(\theta)^q]<\infty$ for all $q>0$ in \eqref{eq:moment-assumption}.
    \end{remark}

    This is a standard law of large numbers for SMC samplers / Feynman--Kac particle systems.
    See, \eg \cite{moral04book-feynman},
    or \cite{delmoral06jrsc-smc}, which specializes the Feynman--Kac theory to static targets with MCMC move steps.
    The above theorem is stated for a predefined tempering schedule. Analogous convergence results also hold under adaptive tempering schemes (\eg ESS-based); see \cite{beskos16aap-adaptive-smc-convergence}.

    While Theorem~\ref{thm:tsmc_consistency} provides a consistency guarantee for TSMC, it is asymptotic in the number of particles and may therefore be of limited practical value at finite $N$. Beyond the choice of $N$, practical performance depends critically on the MCMC kernel, in particular its mixing and its ability to maintain particle diversity. In practice, we monitor diagnostics such as the effective sample size (ESS) and the MCMC acceptance rate; and we tune the parameters of the MCMC kernel to achieve good performance.

%!TEX root = ../main.tex

\section{Supplementary Results and Discussion}
\label{sec:supplementary-experiments}

\subsection{Hyperparameters of TSMC}
\label{subsec:hyperparameters-of-tsmc}

Table~\ref{tab:hyperparameters} summarizes three key hyperparameters used by \tsmc{} in our trajectory optimization (TO) and policy optimization (PO) experiments. Hyperparameters not shown in the table can be found in the code repository. We discuss tuning considerations below.

\begin{itemize}
    \item Temperature ($\lambda$): The temperature controls how sharply the tempered distribution concentrates. Higher temperature produces a more diffuse distribution, while lower temperature concentrates mass more strongly near global minimizers. Two factors primarily determine the choice of $\lambda$:
    \begin{enumerate}
        \item \emph{Scale of the energy.} When energies are on the order of hundreds to thousands (\eg Pendulum, Acrobot, and pusher--slider in TO experiments), a temperature of $0.1$ is usually sufficient. For inverted pendulum with sparse rewards in the PO experiments, the terminal reward is normalized to $[0,1]$, so we use a smaller temperature to obtain sufficient concentration.
        
        \item \emph{Landscape sharpness.} For Acrobot and double pendulum on a cart in the PO experiments, we empirically found that larger temperatures work better: at low temperature, \tsmc{} often gets trapped in poor local minima, suggesting a sharp landscape that benefits from additional smoothing.
    \end{enumerate}
    \item NUTS step size ($\varepsilon$): A larger step size tends to increase exploration and particle movement, improving diversity. However, too large a step size can cause numerical instability in leapfrog integration and lead to near-zero Metropolis--Hastings acceptance. We tune $\varepsilon$ by monitoring the MCMC acceptance rate, and note that feasible step sizes depend on the smoothness/sharpness of the energy landscape. In our TO experiments we can use larger step sizes, whereas in PO we use smaller step sizes to maintain stable acceptance.
    \item Number of particles ($N$): Increasing $N$ improves Monte Carlo accuracy (\cf Theorem~\ref{thm:tsmc_consistency}) at the cost of additional computation. In \jax{}, particle operations are heavily vectorized; in practice, we can often increase $N$ up to GPU memory limits without a proportional slowdown. For TO, we intentionally used $N=100$ to demonstrate that \tsmc{} remains effective with a relatively small particle set. For PO, we used larger particle counts to ensure the algorithm can sample from the entire distribution.
\end{itemize}

Overall, PO tends to be ``harder'' than TO in our experiments: it typically requires higher temperatures, smaller NUTS step sizes, and larger particle counts. Section~\ref{subsec:discussion-to-po} further discusses the TO--PO distinction. 

% For more details on the choice of hyperparameters, please refer to Table~\ref{tab:hyperparameters}.

\begin{table}[h]
    \centering
    \caption{Three key hyperparameters used by \tsmc{} for trajectory optimization and policy optimization experiments.
    \label{tab:hyperparameters}}
    \vspace{-2mm}
    \begin{adjustbox}{max width=\textwidth}
    \begin{tabular}{c|c|c|c|c|c|c}
        \toprule
        & \multicolumn{3}{c|}{Trajectory Optimization} & \multicolumn{3}{c}{Policy Optimization} \\
        \hline
        Parameters & Inverted Pendulum & Acrobot & Pusher Slider & Inverted Pendulum (w/ Sparse Rewards) & Acrobot & Double Pendulum on a Cart \\
        \hline
        Temperature ($\lambda$) & $0.1$ & $0.1$ & $0.2$ & $0.0005$ & $100$ & $200$ \\
        \hline
        NUTS Step Size ($\varepsilon$) & $0.2$ & $0.02$ & $0.01$ & $0.001$ & $0.001$ & $0.001$ \\
        \hline
        Number of Particles ($N$) & $100$ & $100$ & $100$ & $3000$ & $16000$ & $14000$ \\
        \bottomrule
    \end{tabular}
    \end{adjustbox}
\end{table}

\subsection{Runtime of TSMC and Baselines}

Table~\ref{tab:runtime-to-experiments} reports the runtime of \tsmc{} and the baselines for trajectory optimization. \pmppi{} is the fastest method since it does not differentiate through rollouts. \tsmc{} is slower than \pnuts{} and \pmala{} due to the tempering procedure. \pipopt{} is reasonably fast on the low-dimensional inverted pendulum problem, but becomes prohibitively slow on pusher--slider due to high dimensionality and nonsmooth dynamics. Overall, \tsmc{} trades additional compute for lower-cost solutions, consistent with Fig.~\ref{fig:to-figure} in the main paper.

Table~\ref{tab:runtime-po-experiments} reports runtime for policy optimization. \tsmc{} is generally comparable to the baselines, though in many cases it is about \(3\times\) slower. Nevertheless, Fig.~\ref{fig:po-figure} in the main paper shows that \tsmc{} attains better policies. Finally, compared with \ppo{} and \sac{}, \tsmc{} does not learn a value function, which simplifies the overall pipeline.

In summary, \tsmc{} can be slower than standard baselines due to its algorithmic complexity, but this overhead can be worthwhile given its robust performance.

\begin{table}[h]
    \centering
    \caption{Runtime of \tsmc{} and baselines (in seconds) for trajectory optimization experiments.
    \label{tab:runtime-to-experiments}}
    \vspace{-2mm}
    \begin{adjustbox}{max width=\textwidth}
    \begin{tabular}{c|c|c|c}
        \toprule
        Methods & Inverted Pendulum & Acrobot & Pusher Slider \\
        \hline
        \tsmc & $140.64$ & $1627.47$ & $2033.36$  \\
        \hline
        \pnuts & $24.69$ & $109.15$ & $1248.52$ \\
        \hline
        \pmala & $12.44$ & $543.72$ & $526.79$  \\
        \hline
        \pmppi & $1.86$ & $5.09$ & $12.10$ \\
        \hline
        \pipopt & $15.97$ & $1998.14$ & $55391.26$ \\
        \bottomrule
    \end{tabular}
    \end{adjustbox}
\end{table}

\begin{table}[h]
    \centering
    \caption{Runtime of \tsmc{} and baselines (in seconds) for policy optimization experiments.
    \label{tab:runtime-po-experiments}}
    \vspace{-2mm}
    \begin{adjustbox}{max width=\textwidth}
    \begin{tabular}{c|c|c|c|c|c|c}
        \toprule
        Methods & Inverted Pendulum (w/ Sparse Rewards) & Acrobot & Double Pendulum on a Cart \\
        \hline
        \tsmc & \tsmcd: $1514.19 \quad$ \tsmce: $2035.97$ & $554.55$ & $1080.07$ \\
        \hline
        \ppo & $414.45$ & $719.26$ & $303.83$ \\
        \hline
        \sac & $304.56$ & $347.61$ & $224.02$ \\
        \bottomrule
    \end{tabular}
    \end{adjustbox}
\end{table}

\begin{figure}[h!]
    \centering
    \includegraphics[width=\textwidth]{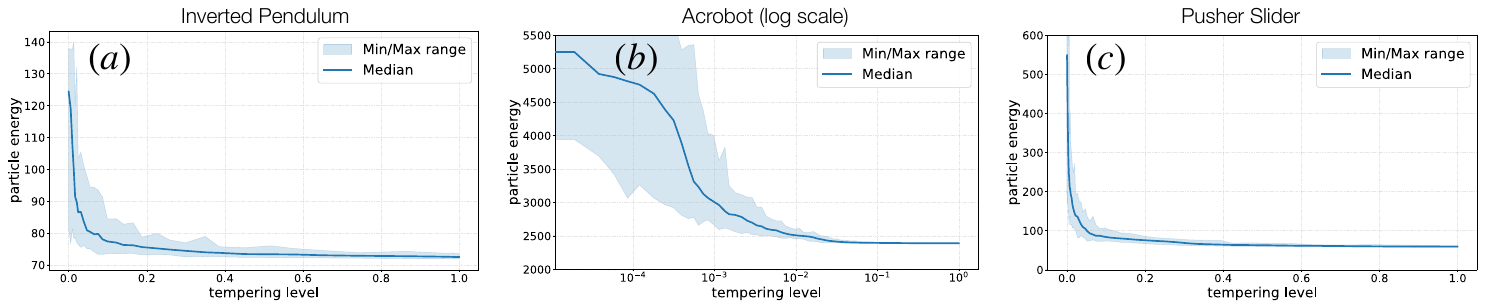}
    \vspace{-7mm}
    \caption{Energy distributions across tempering steps of \tsmc{} for the trajectory optimization experiments. (a) inverted pendulum, (b) acrobot, (c) pusher slider. Note that the $x$-axis for the acrobot results is in log scale for better visualization.
    \label{fig:energy-vs-tempering}}
    \vspace{-4mm}
\end{figure}

\subsection{Energy Distributions across Tempering Steps}

Fig.~\ref{fig:po-figure} in the main paper reports energy distributions across tempering steps for the policy optimization experiments. Fig.~\ref{fig:energy-vs-tempering} shows the analogous plots for trajectory optimization. As tempering progresses, the distributions concentrate and shift toward lower energies, indicating convergence to good solutions of the original problems.

\subsection{Discussion: Trajectory Optimization vs. Policy Optimization}
\label{subsec:discussion-to-po}

The interested reader may have noticed that, in the policy optimization experiments for Acrobot and double pendulum on a cart, we use a single-point initial-state distribution. This matches the trajectory optimization setup, except that we now optimize a parameter vector $\theta$ (the weights of a neural network policy) rather than an open-loop control sequence. As noted in Section~\ref{subsec:hyperparameters-of-tsmc}, this seemingly minor change makes PO empirically harder than TO. What is the explanation?

The key lies in Theorem~\ref{thm:gradient-computation-trajectory-optimization} in the main paper. From \eqref{eq:grad-J}, the gradient of the trajectory cost with respect to $\theta$ depends on two terms, $G_t$ and $g_t$. These terms behave very differently in the TO and PO settings.

\begin{itemize}
    \item $G_t$: Recall the definition of $G_t$ in \eqref{eq:gradient-dynamics-controller}. For TO, $G_t$ simply selects the $t$-th control block, \ie $G_t=[0,\ldots,I_m,\ldots,0]\in\Real{m\times Tm}$. For PO, by contrast, $G_t$ captures the sensitivity of the control prediction with respect to the policy parameters $\theta$. While $G_t$ is well conditioned in TO, it can become ill conditioned in PO if the policy has a large Lipschitz constant (\ie a small perturbation in $\theta$ induces a large change in $u_t$).
    \item $g_t$: Recall the definition of $g_t$ in \eqref{eq:adjoint-recursion}. The TO and PO recursions differ only through the term $L_t$ in \eqref{eq:gradient-dynamics-controller}. In TO, $L_t=0$ because the control is not a function of the state. In PO, $L_t$ measures the sensitivity of the control prediction with respect to the state $x_t$. This sensitivity can be extremely large when the optimal policy is nonsmooth (\eg bang-bang), as in inverted pendulum and Acrobot, where a small change in state can flip the control from $-u_\max$ to $+u_\max$; see, \eg \cite{han24l4dc-pendulum}.
\end{itemize}

In summary, both $G_t$ and $g_t$ can attain large magnitudes in PO, whereas they are well conditioned in TO. This can amplify $\nabla_\theta J$ in PO and lead to a sharper energy landscape than in TO. For this reason, we used higher temperatures in PO to smooth out the energy landscape (\cf Table~\ref{tab:hyperparameters}).

This analysis also sheds light on why \ppo{} and \sac{} perform poorly on Acrobot and double pendulum on a cart. A recent result \cite{han26icml-nsr} relates the variance of policy-gradient estimators to terms like $g_t$: when $g_t$ has large magnitude, gradient estimates can have high variance. In these tasks, $g_t$ can indeed be large due to the nonsmoothness of near-optimal policies, which can make learning unstable within a finite interaction budget.

%!TEX root = ../main.tex

\section{Related Work}
\label{sec:related-work}

We briefly review related work on trajectory optimization, policy optimization, sampling from the Boltzmann-tilted distribution, and annealing and tempering.

\subsection{Trajectory Optimization}

\textbf{Optimization-based Methods.} Nonlinear programming (NLP) methods \cite{nocedal06book-numerical}---in either single-shooting or multiple-shooting form---are among the most widely used approaches to trajectory optimization in control and robotics. Notable examples include interior-point methods \cite{wachter06mp-ipopt}, sequential quadratic programming \cite{gill05siam-snopt}, sequential convex programming \cite{mao16cdc-successive,li25rss-crisp}, and augmented Lagrangian methods \cite{howell19iros-altro,jallet22arxiv-proxnlp}. These are local methods that refine an initial guess and can get trapped in undesirable local minima. In contrast, global solvers avoid dependence on a good initialization but are typically far more computationally expensive; common approaches include mixed-integer programming (branch-and-bound/cut) \cite{deits15icra-efficient,marcucci24siopt-shortest} and convex relaxations \cite{kang24wafr-strom,kang25rss-spot,han25rss-xm,Yang24book-sdp}. Most optimization-based methods assume that the dynamics are analytically specified or that accurate gradients are available.

\textbf{Sampling-based Methods.} Sampling-based methods draw many candidate control sequences from a proposal distribution and then aggregate them to minimize trajectory cost. The simplest strategy is to select the best sample (predictive sampling) \cite{howell22arxiv-predictive,qi25arxiv-gpc}. Model predictive path integral control (MPPI) \cite{williams17jgcd-mppi,williams15arxiv-mppi} instead uses an exponential weighting of trajectory costs, which can be viewed as importance sampling for the Boltzmann-tilted distribution. The cross-entropy method (CEM) \cite{de05aor-cem} iteratively refines the proposal by fitting it to elite samples. More recently, \cite{pan24neurips-model} learns a score function from rollouts to enable model-based sampling. These methods are naturally parallelizable and can be applied with black-box dynamics, but often require many samples to achieve strong performance.

TSMC differs from both strands by combining global exploration from sampling-based methods with gradient-based local refinement from optimization-based methods.

\subsection{Policy Optimization}

\textbf{Zero-order Methods.} At the heart of modern policy optimization in reinforcement learning lies the policy gradient lemma \cite{sutton98book-rlbook,sutton99neurips-policygradient}, which enables gradient-based learning of policy parameters directly from sampled trajectories without requiring an explicit model of the environment. Despite their generality, policy-gradient methods can be unstable and converge to poor solutions. Prior work has identified multiple sources of this instability. One is the \emph{high variance} of stochastic policy-gradient estimators, which can cause updates to poorly track the true ascent direction; this motivates variance-reduction techniques such as baselines, generalized advantage estimation, and actor--critic methods~\cite{sutton99neurips-policygradient,greensmith04jmlr-variancereduction,schulman15iclr-gae,haarnoja18icml-sac}. 
Another is overly aggressive policy updates, motivating methods that explicitly constrain update size, \eg via Kullback--Leibler trust regions in TRPO~\cite{schulman15icml-trpo} or clipping in PPO~\cite{schulman17arxiv-ppo}. In contrast to optimizing a single parameter vector $\theta$, evolution strategies~\cite{salimans17arxiv-es} optimize a distribution over policy parameters, often approximated by a Gaussian whose mean is updated.

\textbf{First-order Methods.} In contrast to zero-order methods that estimate policy gradients from sampled trajectories, first-order methods exploit differentiability of the dynamics/simulator to compute gradients more efficiently. Notable examples include \cite{clavera20iclr-model,xu22iclr-accelerated,amigo25corl-forl,amos21l4dc-model,wiedemann23icra-apg}. While these methods are receiving increasing attention in robotics, a key challenge remains contact-rich tasks, where computing or approximating useful dynamics gradients is nontrivial.

TSMC is related to both lines of work. Like evolution strategies, it optimizes a distribution over parameters rather than a single $\theta$. Like first-order methods, it leverages differentiable dynamics within its MCMC rejuvenation steps. The key difference is that TSMC adopts principled machinery from Bayesian computation---sampling from a Boltzmann-tilted distribution via tempering---to better handle sharp, multimodal targets. Moreover, TSMC provides a unified framework for both trajectory optimization and policy optimization.

\subsection{Sampling from the Boltzmann-tilted Distribution}

Sampling from a Boltzmann tilt has a long history in statistical physics (Gibbs measures) and enters control via the ``control-as-inference'' viewpoint, where optimal control is posed as inference in an exponential-family model and the solution takes the form of exponentiated-cost reweighting of a prior over actions, trajectories, or policy parameters. Beyond importance sampling, Markov chain Monte Carlo (MCMC) and its tempered/sequential variants \cite{neal01sc-ais,doucet01book-smc,delmoral06jrsc-smc,dai22jasa-smc} are widely used for such targets because they help maintain diversity and mitigate particle degeneracy. We use \blackjax{} \cite{cabezas24arxiv-blackjax}, a \jax{} library that implements these methods.

In parallel, there is growing interest in \emph{neural samplers} (a.k.a.\ learned samplers, amortized MCMC, or neural transport) that replace some or all of the sampling machinery with trainable maps. Representative examples include normalizing flows \cite{albergo19prd-flow}, denoising diffusion samplers \cite{vargas23arxiv-denoising},  path integral sampler \cite{zhang21arxiv-path}, iterated denoising energy matching \cite{akhound24arxiv-iterated}, and adjoint sampling \cite{havens25arxiv-adjoint}.
In this work, we apply classical TSMC to trajectory and policy optimization. A natural future direction is to investigate whether neural samplers can outperform classical TSMC in these settings.

\subsection{Annealing and Tempering}

Annealing and tempering address the challenge of sharp and/or multimodal objectives by introducing a temperature schedule that gradually deforms an easy distribution into a target distribution. In nonconvex optimization, this viewpoint underlies simulated annealing \cite{kirkpatrick83science-simulated} and deterministic annealing \cite{rose02ieee-deterministic}, which start from smooth, high-temperature (or entropy-regularized) objectives and progressively sharpen them, often improving robustness to poor local minima. In robotics and computer vision, graduated non-convexity \cite{yang20ral-gnc} similarly solves a sequence of problems that deform from a convex surrogate to a nonconvex, outlier-robust objective. In sampling, tempering improves mixing across modes: parallel tempering (replica exchange) runs chains at multiple temperatures and swaps states \cite{swendsen86prl-replica}, while tempered transitions build global proposals by traversing a temperature ladder \cite{neal96sc-sampling}. For population-based estimators, annealed importance sampling bridges from a tractable base distribution to a target to reduce variance and estimate normalizing constants \cite{neal01sc-ais}, and sequential Monte Carlo samplers strengthen this idea via resampling and MCMC move steps in the canonical reweight--resample--move template \cite{delmoral06jrsc-smc}.

\clearpage
\pagestyle{empty}
% \nocite{*} % remove once you have real \cite{...} calls
\printbibliography

\end{document}